\documentclass[letterpaper]{article}
\usepackage[preprint]{aaai2027}
\usepackage{helvet}
\usepackage{courier}
\usepackage[hyphens]{url}
\usepackage{graphicx}
\usepackage{natbib}
\usepackage{caption}
\usepackage{booktabs}
\usepackage{amsmath, amssymb, amsfonts, amsthm}
\usepackage{nicefrac}
\usepackage{mathtools}
\usepackage{xcolor}
\usepackage{array, multirow, makecell}
\usepackage{adjustbox}
\usepackage{enumitem}
\usepackage{float}

\theoremstyle{plain}
\newtheorem{theorem}{Theorem}
\newtheorem{lemma}[theorem]{Lemma}
\newtheorem{corollary}[theorem]{Corollary}
\newtheorem{proposition}[theorem]{Proposition}
\newtheorem{algorithm}[theorem]{Algorithm}
\theoremstyle{definition}
\newtheorem{definition}[theorem]{Definition}
\theoremstyle{remark}

\def\ie{\textit{i.e.,~}}
\def\eg{\textit{e.g.,~}}

\title{$x$-Prediction Is All You Need: \\ Training-Free Accelerated Generation via Endpoint Decodability}

\author{
    Xin Peng\textsuperscript{\rm 1}\quad
    Ang Gao\textsuperscript{\rm 1}\thanks{Corresponding author}
}
\affiliations{
    \textsuperscript{\rm 1}School of Physical Science and Technology,\\
    Beijing University of Posts and Telecommunications, Beijing, China\\
    \texttt{anggao@bupt.edu.cn}
}

\begin{document}
\maketitle

\begin{abstract}
Diffusion and flow matching models generate high-quality samples, but their ODE samplers often need tens to hundreds of neural function evaluations (NFEs). This remains a practical challenge for released checkpoints, since many accelerators require additional design choices and training cost through retraining, distillation, or trajectory redesign. We investigate a different route based on $x$-prediction. During sampling, standard affine probability paths already expose $x_0$ information: an intermediate state and its path velocity determine a principled estimate of the clean sample. We formalize this property as \textbf{endpoint decodability} and show that the decoder is the minimum-MSE estimator $\mathbb{E}[x_0\mid x_t]$ under the usual $\ell_2$ objective. This yields \textbf{Truncated Jump Sampling} (TJS): stop the ODE at an early-exit time $t^*$ and return the decoded $x_0$. TJS requires no retraining, distillation, or architecture change. Across SDXL, SD3.5M, Z-Image-Turbo, and three class-conditional benchmarks, it reduces NFEs by 20--70\% with near-matched quality. The analysis also shows why endpoint prediction can work without straightening the trajectory, providing inference acceleration without trajectory redesign.
\end{abstract}

\section{Introduction}

Diffusion models~\cite{ho2020ddpm, song2021scorebased, nichol2021improved} and flow matching~\cite{lipman2022flow, albergo2023building, liu2023flow} produce stunning images but require tens to hundreds of neural function evaluations (NFEs) per sample~\cite{rombach2022latent, esser2024scaling, flux2024, peebles2023dit}. Reducing this cost is a central challenge in generative modeling.

The field has responded with a range of solutions---Rectified Flow~\cite{liu2023flow}, distillation (progressive~\cite{salimans2022progressive}, consistency-based~\cite{song2023consistency, luhman2023latent, zheng2023trajectory}, bootstrapping~\cite{gu2023boot}), advanced solvers~\cite{lu2022dpm, lu2022dpmpp, zhang2023fast}, and training-time scheduling~\cite{densejump}---all effective at reducing NFEs, all requiring changes to the trajectory, model, or training recipe. \textbf{No existing approach achieves acceleration purely by modifying inference.}

This complexity has a real and growing cost. The community has produced an enormous library of pretrained checkpoints---SDXL alone has thousands of fine-tuned variants on CivitAI and Hugging Face~\cite{podell2023sdxl, esser2024scaling, flux2024, rombach2022latent, peebles2023dit}. While fast ODE solvers (DPM-Solver, UniPC) can reduce the step count, further acceleration without retraining remains an open practical challenge. \textbf{The field would benefit from a training-free acceleration strategy that works on models as they are, not as they could be after retraining.}

There is a simpler path, hiding in plain sight within the training objective itself. Diffusion and flow matching models are trained to predict $x_0$ at every timestep. During inference, the model thus produces a valid $x_0$ estimate at \emph{every intermediate step}---and these estimates improve as integration proceeds. Once good enough, why continue? For diffusion models, DDIM~\cite{song2021ddim} already computes $\hat{x}_0$ internally, and ODE-Jump~\cite{chen2024geometric} previously proposed integrating the probability-flow ODE to an intermediate time and then returning the denoising output as the final sample in the VE/EDM setting; here we formulate this endpoint-jump principle for general non-degenerate affine paths, including flow matching, and analyze when and why it works. We term this property \textbf{endpoint decodability}: for non-degenerate affine paths, the intermediate state and path velocity recover $x_0$ through a closed-form decoder, which we formalize in Section~\ref{sec:endpoint_decodability}. Crucially, under the standard $\ell_2$ loss, this decoder is the MMSE-optimal estimator $\mathbb{E}[x_0|x_t]$ (Theorem~\ref{thm:mmse_optimality}). The practical consequence: stop when the estimate is sufficient, output it. We call this \textbf{Truncated Jump Sampling (TJS)}.

\begin{figure*}[t]
  \centering
  \includegraphics[width=0.85\linewidth]{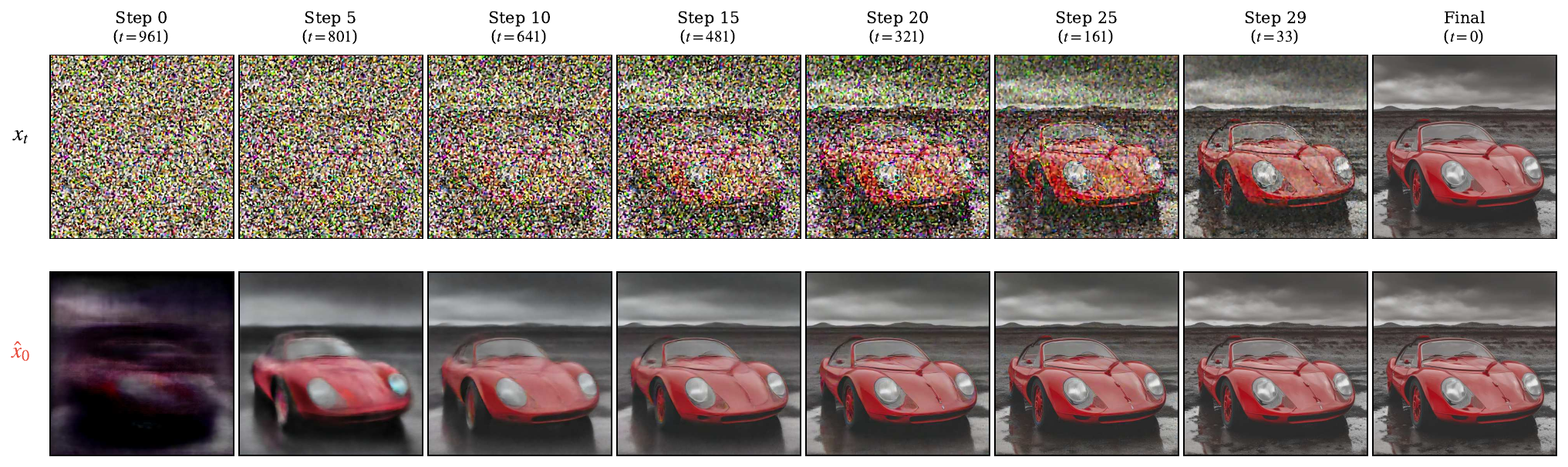}
  \caption{Endpoint decodability in action. Top: $x_t$ decoded directly (noisy at early steps). Bottom: $x_0$ via endpoint decoding (clean at any step).}
  \label{fig:xt_vs_x0}
\end{figure*}

TJS requires no retraining, no distillation, no architecture changes. It works on any pretrained checkpoint---SDXL, SD3.5M, Z-Image-Turbo, FLUX, DiT, fine-tuned variants---without modification. TJS does not replace distillation for extreme few-step generation (1--4 NFE); it provides a complementary, zero-cost path for the moderate regime (30 $\to$ 15--25 steps). \textbf{Immediate, training-free acceleration for models you already have.}

We make three contributions:
\begin{enumerate}[leftmargin=*,nosep]
    \item \textbf{Principle and theory.} We formalize \textbf{endpoint decodability}: for any non-degenerate affine path, $(x_t,u_t)$ recovers $x_0$ through a closed-form decoder (Theorem~\ref{thm:endpoint_decodability}), the induced decoder is MMSE-optimal (Theorem~\ref{thm:mmse_optimality}), its error is curvature-independent (Theorem~\ref{thm:error_decomposition}), and it can strictly beat coarse Euler (Theorem~\ref{thm:euler_comparison}). All standard parameterizations are equivalent at optimality (see Supplementary Material, \S A). Critically, straight trajectories are \emph{sufficient but not necessary} (Proposition~\ref{prop:straightness}), challenging the foundation of Rectified Flow and Consistency Models.

    \item \textbf{Algorithm.} \textbf{Truncated Jump Sampling (TJS)}: a training-free early-exit sampler---zero retraining, zero distillation, zero architecture changes. TJS is compatible with any ODE solver, noise schedule, CFG scale, or distillation method.

    \item \textbf{Experiments.} Across six model families---SDXL, SD3.5M, Z-Image-Turbo~\cite{team2025zimage}, ImageNet-256, CIFAR-10, MNIST---TJS reduces NFE by 20--70\% at near-matched quality, with strict monotonic improvement on all metrics.
\end{enumerate}

\section{Related Work}

\paragraph{Diffusion and Flow Matching.}
Diffusion models~\cite{ho2020ddpm, song2021scorebased, nichol2021improved} and flow matching~\cite{lipman2022flow, albergo2023building, liu2023flow} share a common skeleton: probability paths $x_t = \alpha_t x_0 + \sigma_t \epsilon$ interpolating noise to data. DDIM~\cite{song2021ddim} and EDM~\cite{karras2022elucidating} established deterministic sampling via neural ODEs~\cite{chen2018neuralode}. These models power state-of-the-art image generation---SD3~\cite{esser2024scaling}, FLUX~\cite{flux2024}, DiT~\cite{peebles2023dit}. \textbf{Every model cited above uses the same affine path structure.} Endpoint decodability requires nothing else.

\paragraph{Few-Step Generation and Trajectory Straightening.}
Reducing NFEs spans several paradigms, all training-stage. Distillation (progressive~\cite{salimans2022progressive}, masked~\cite{meng2022distillation}, bootstrapping~\cite{gu2023boot}, trajectory consistency~\cite{zheng2023trajectory}) and Consistency Models~\cite{song2023consistency, luhman2023latent} train students for few-step generation. Rectified Flow~\cite{liu2023flow} learns straight paths via reflow (Euler error $\propto \|\ddot{x}_t\|$). Dense-Jump FM~\cite{densejump} modifies inference trajectories. Higher-order solvers (DPM-Solver~\cite{lu2022dpm}, DPM-Solver++~\cite{lu2022dpmpp}, PNDM~\cite{liu2022pseudo}, UniPC~\cite{zhang2023fast}) improve discretization accuracy but still traverse the full trajectory. The shared motivation: few-step quality benefits from modifying training. TJS offers a complementary, training-free alternative exploiting an inherent structural property of affine paths.

\paragraph{Prediction Parameterizations and Theoretical Connections.}
$x_0$-, $v$-, $\epsilon$-, and score-prediction~\cite{karras2022elucidating, salimans2022progressive, ho2020ddpm, song2021scorebased} are typically treated as distinct choices with different SNR trade-offs~\cite{kingma2021vdm}. \textbf{Our framework reveals their algebraic equivalence at optimality} (see Theorem 13 in the Supplementary Material, \S A), so any pretrained model encodes an endpoint predictor. Complementary evidence: JiT~\cite{li2025jit} shows $x_0$-prediction dramatically outperforms $\epsilon$/v-prediction on raw-pixel patches via manifold structure; MeanFlow~\cite{geng2025meanflow} enables one-step generation via average velocity. Beyond sampling, flow matching connects to optimal transport~\cite{lipman2022flow, tong2023conditional} and the Schr\"odinger bridge~\cite{de2021diffusion}; I-MMSE~\cite{guo2005mmse} provides information-theoretic grounding for $\mathcal{U}(t)$ (Theorem 14 in the Supplementary Material, \S A).

\section{Preliminaries}

Diffusion and flow matching models share a simple but powerful structure: they define paths that morph noise into data. We now formalize this structure and preview the theory that follows. We establish three results in order: (1) the algebraic condition under which any intermediate state can recover the clean endpoint (Theorem~\ref{thm:endpoint_decodability}); (2) the optimality of this recovery under standard $\ell_2$ training (Theorem~\ref{thm:mmse_optimality}); and (3) a decomposition of TJS error into two curvature-independent components (Theorem~\ref{thm:error_decomposition}). Together, these results explain why straight trajectories---the central requirement of prior work---are sufficient but not necessary.

\begin{definition}[Affine Probability Path]
\label{def:affine_path}
An \emph{affine probability path} is a one-parameter stochastic process
\begin{equation}
    x_t = \alpha_t x_0 + \sigma_t \epsilon, \quad \epsilon \sim \mathcal{N}(0,\mathbf I), \quad t\in[0,1],
    \label{eq:affine_path}
\end{equation}
where $x_0 \sim p_{\mathrm{data}}$ is the clean sample, and $\alpha_t, \sigma_t: [0,1] \to \mathbb{R}_{\ge 0}$ are $C^1$ functions with $\alpha_0=0,\sigma_0=1$ and $\alpha_1=1,\sigma_1=0$.
\end{definition}

This single equation covers VP/VE diffusion, EDM, and linear flow matching. The entire diversity of modern generative models---billions of parameters, dozens of training tricks---reduces to two scalar schedules $\alpha_t$ and $\sigma_t$. Endpoint decodability is a property of these two functions alone.

\begin{lemma}[Conditional Velocity]
\label{lem:conditional_velocity}
The conditional velocity field $u_t \coloneqq \frac{\mathrm d x_t}{\mathrm d t}$ satisfies
\begin{equation}
    u_t = \dot{\alpha}_t x_0 + \dot{\sigma}_t \epsilon.
    \label{eq:conditional_velocity}
\end{equation}
\end{lemma}

Combining Eqs.~\ref{eq:affine_path}--\ref{eq:conditional_velocity} yields a coupled linear system:
\begin{equation}
    \begin{bmatrix} x_t \\ u_t \end{bmatrix}
    = \begin{bmatrix} \alpha_t & \sigma_t \\ \dot{\alpha}_t & \dot{\sigma}_t \end{bmatrix}
    \begin{bmatrix} x_0 \\ \epsilon \end{bmatrix}.
    \label{eq:linear_system}
\end{equation}

\begin{definition}[Path Determinant]
The \emph{path determinant} is $\Delta_t \coloneqq \det\begin{bmatrix} \alpha_t & \sigma_t \\ \dot{\alpha}_t & \dot{\sigma}_t \end{bmatrix} = \dot{\alpha}_t \sigma_t - \alpha_t \dot{\sigma}_t$.
\end{definition}

\section{Endpoint Decodability}
\label{sec:endpoint_decodability}

The coupled linear system (Eq.~\ref{eq:linear_system}) contains the entire story. If the $2\times 2$ coefficient matrix is invertible, then $(x_t, u_t)$ uniquely determines $(x_0, \epsilon)$---and in particular $x_0$. The invertibility condition is remarkably simple: the path determinant must be nonzero. We now formalize this observation and prove that the resulting decoder is optimal.

\subsection{Characterization}

Solving the linear system via Cramer's rule (see Supplementary Material, \S A) yields the endpoint decoder:

\begin{definition}[Endpoint Decodability]
\label{def:endpoint_decodability}
An affine probability path is \emph{endpoint-decodable} at time $t \in (0,1]$ if the mapping $(x_t, u_t) \mapsto x_0$ is well-defined and unique. A path is \emph{globally endpoint-decodable} if this holds for all $t \in (0,1]$.
\end{definition}

\begin{theorem}[Endpoint Decodability]
\label{thm:endpoint_decodability}
An affine probability path is endpoint-decodable at $t$ iff $\Delta_t \neq 0$, with
\begin{equation}
    x_0 = \frac{\sigma_t u_t - \dot{\sigma}_t x_t}{\Delta_t}.
    \label{eq:general_endpoint_decoder}
\end{equation}
Given a learned velocity $v_\theta(x_t,t) \approx u_t$, the induced endpoint predictor is $\hat{x}_0^{\mathrm{vel}} = (\sigma_t v_\theta - \dot{\sigma}_t x_t)/\Delta_t$.
\end{theorem}

\begin{proof}[Proof]
See Supplementary Material, \S A. \hfill $\square$
\end{proof}

All common schedules satisfy $\Delta_t \neq 0$ everywhere except possibly at the degenerate boundary $t=0$ (see Supplementary Material, \S A for verification of VP diffusion, VE/EDM, and linear FM). Thus, every standard diffusion or flow matching model is globally endpoint-decodable.

\subsection{MMSE Optimality}

Does plugging a \emph{learned} velocity into the algebraic decoder introduce bias? We prove the answer is no---at optimality, the decoder is Bayes-optimal. Under the standard Flow Matching loss $\mathcal{L}_{\mathrm{FM}}(\theta) = \mathbb{E}_{t, x_0, \epsilon}[\|v_\theta(x_t, t) - u_t\|^2]$, the Bayes-optimal predictor is $v^\star(x,t) = \mathbb{E}[u_t \mid x_t = x]$.

\begin{theorem}[MMSE Optimality of Endpoint Decoder]
\label{thm:mmse_optimality}
Let $v^\star(x, t) = \mathbb{E}[u_t \mid x_t = x]$. Then the induced endpoint predictor recovers the minimum mean square error estimator:
\begin{equation}
    \frac{\sigma_t v^\star(x, t) - \dot{\sigma}_t x}{\Delta_t} = \mathbb{E}[x_0 \mid x_t = x].
    \label{eq:mmse_endpoint}
\end{equation}
\end{theorem}

\begin{proof}[Proof]
See Supplementary Material, \S A. \hfill $\square$
\end{proof}

The practical implication: \textbf{standard training implicitly learns endpoint prediction} in the $\ell_2$ sense. The model is never explicitly trained to predict $x_0$, yet its velocity predictions algebraically encode $\mathbb{E}[x_0|x_t]$---the optimal point predictor under mean squared error. (This is $\ell_2$-optimality, not a claim about perceptual metrics like FID/ImageReward.) The same holds for noise and score prediction (Tweedie's formula~\cite{efron2011tweedie}; see Theorem 13 in the Supplementary Material, \S A).

\textbf{Unified parameterization.} Any pretrained model encodes an endpoint predictor regardless of its output type: all four standard parameterizations are equivalent at optimality (see Theorem 13 in the Supplementary Material, \S A). Direct $x_0$-prediction is optimal for TJS (no algebraic amplification in low-SNR); velocity- and noise-derived decoders work immediately for existing checkpoints.

Figures~\ref{fig:x0_class}--\ref{fig:pareto_tradeoff} preview the experimental evidence: endpoint predictions are clean from early steps, quality improves monotonically, and the speed-quality trade-off follows the predictions of our theory.

\begin{figure*}[t]
  \centering
  \includegraphics[width=0.30\linewidth]{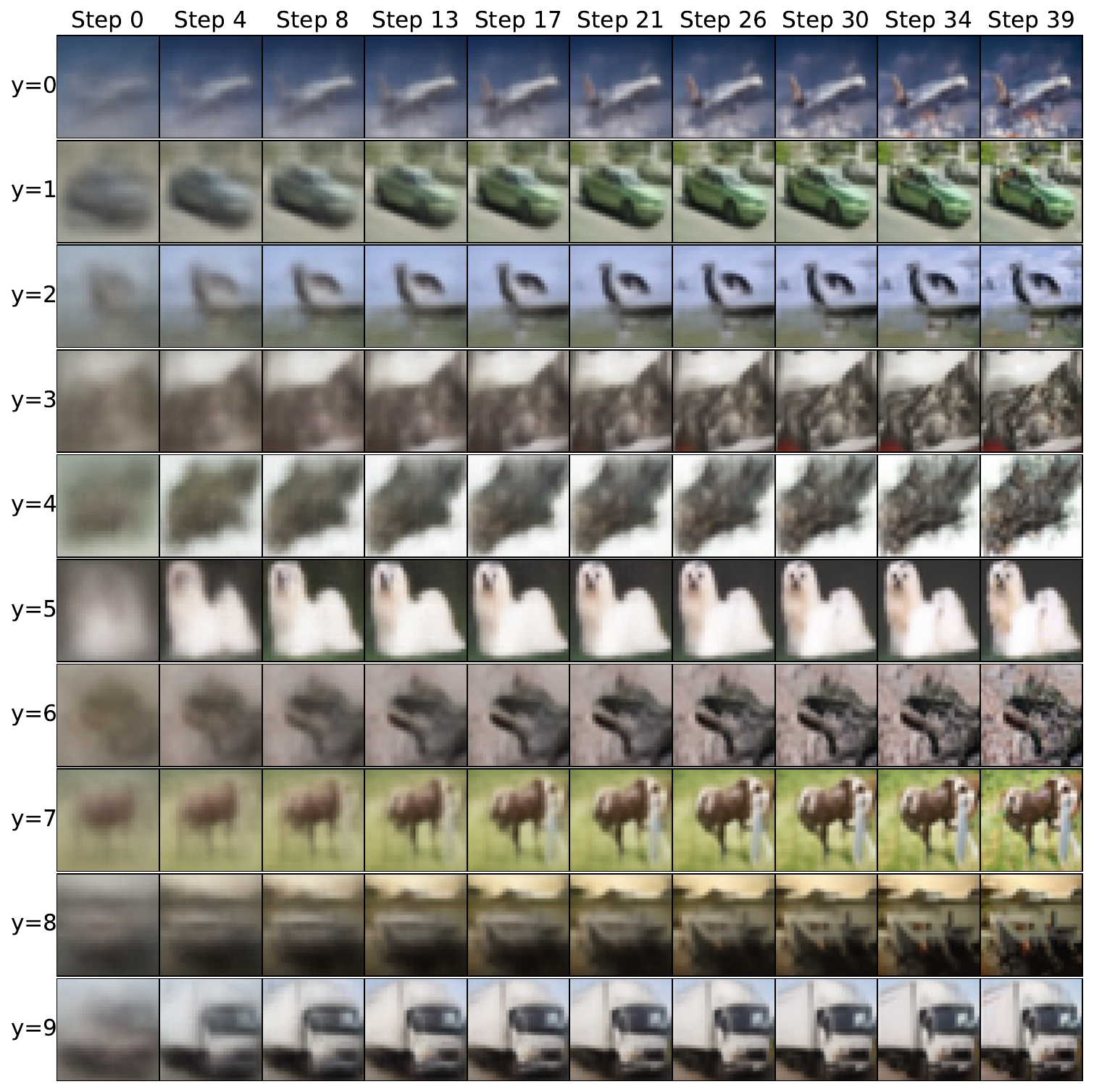}\hfill
  \includegraphics[width=0.30\linewidth]{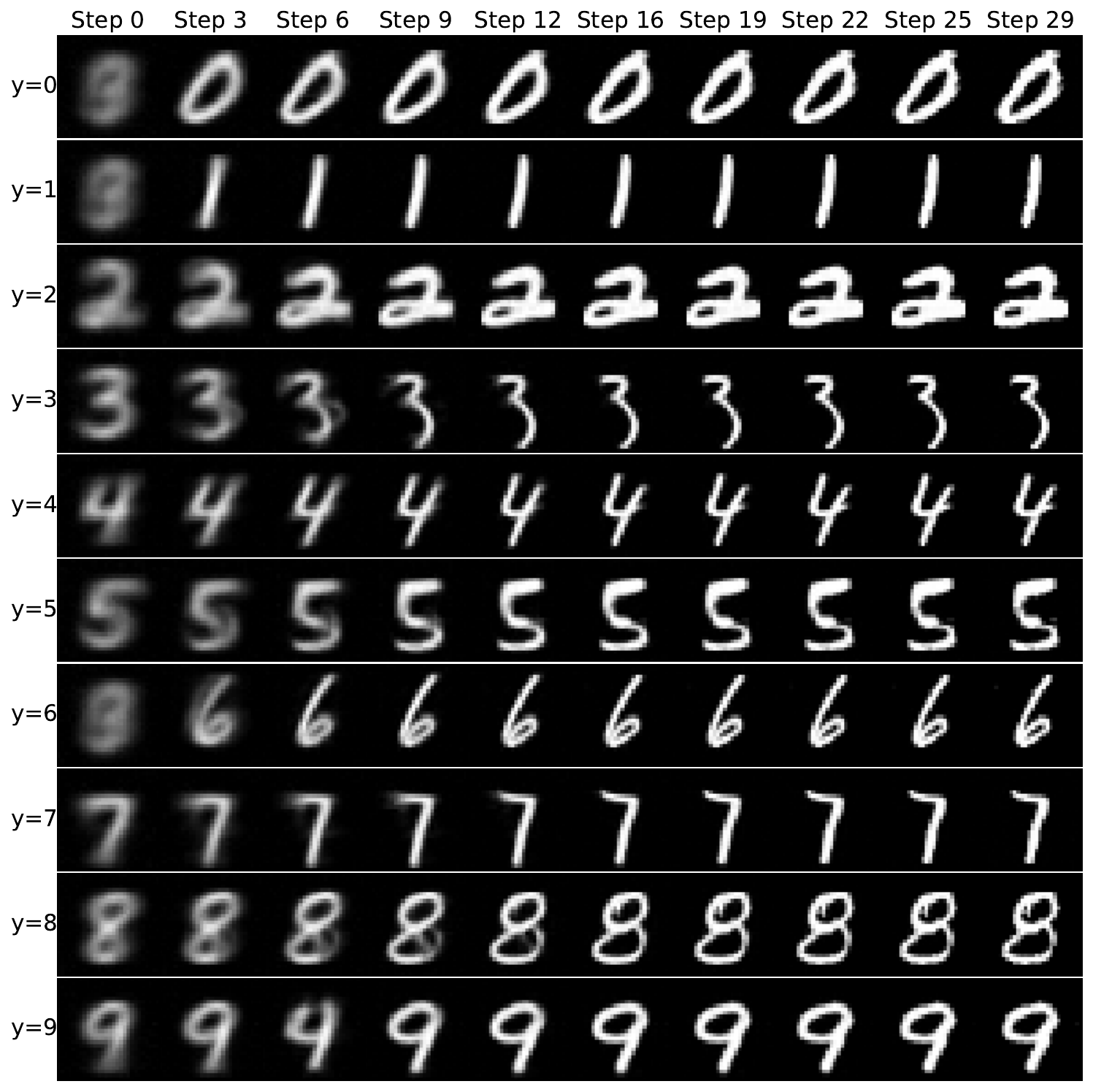}\hfill
  \includegraphics[width=0.30\linewidth]{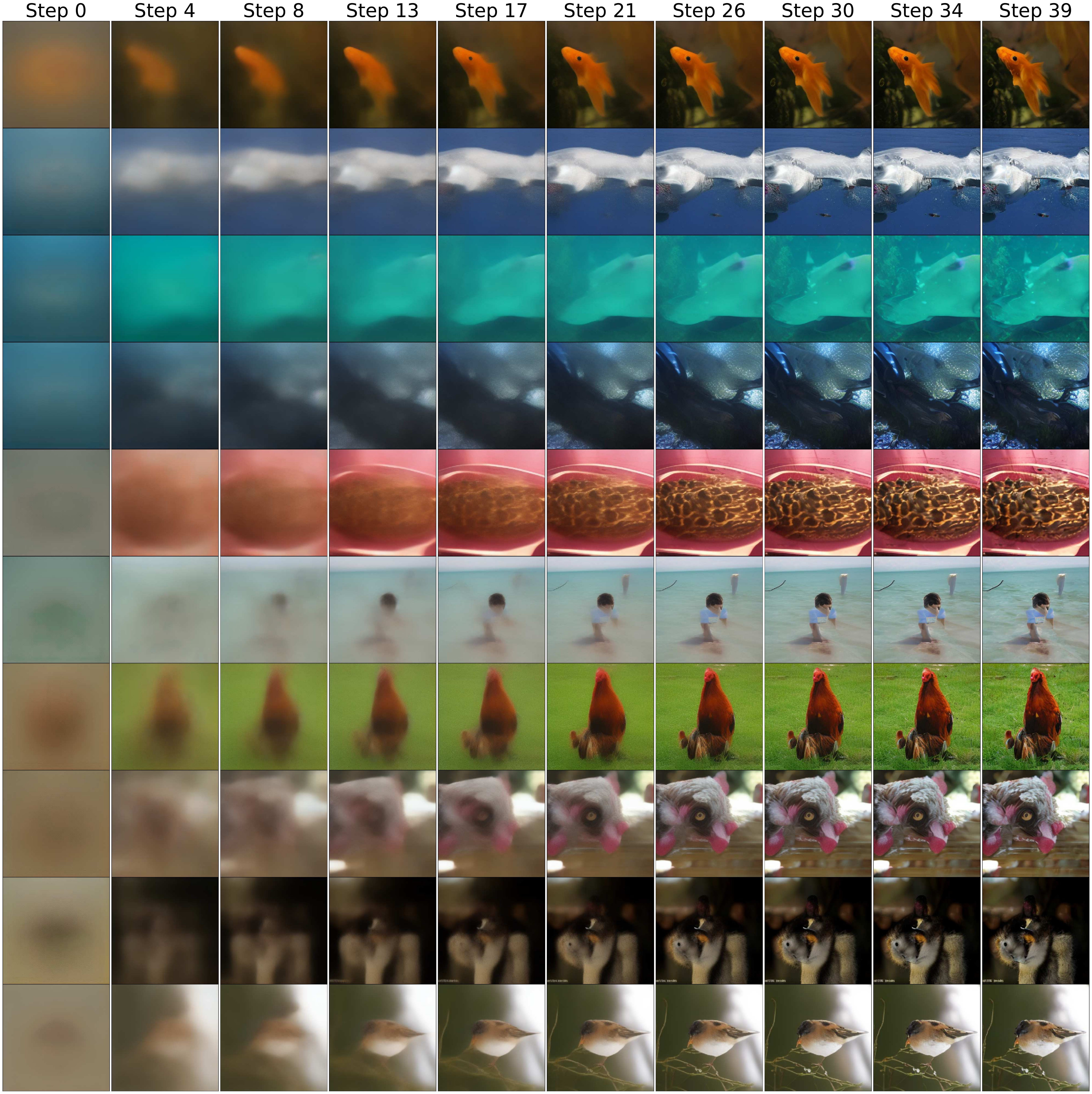}
  \caption{Visual $x_0$ predictions for CIFAR-10 (left), MNIST (center), and ImageNet-256 (right). MNIST saturate at $k^*{\approx}16$ (43\% NFE saving); CIFAR-10/ImageNet-256 at $k^*{\approx}26$ (33\%).}
  \label{fig:x0_class}
\end{figure*}
\begin{figure*}[t]
  \centering
  \includegraphics[width=0.80\linewidth]{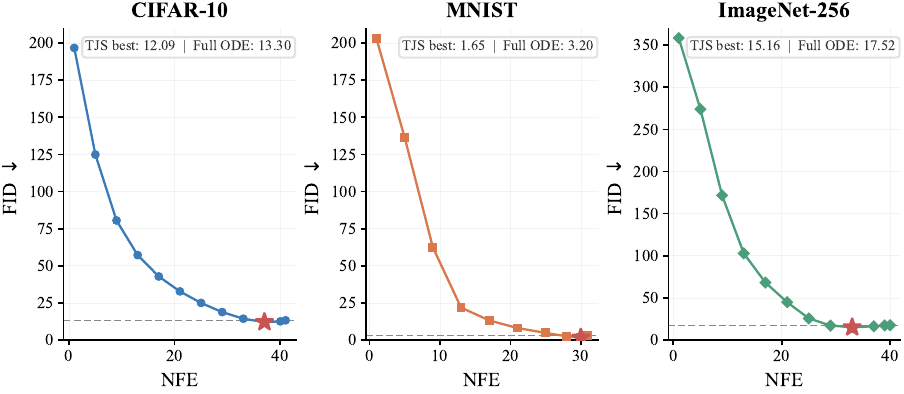}
  \caption{FID vs.\ NFE for TJS on MNIST (30-step) and CIFAR-10/ImageNet-256 (40-step, CFG=1.0). $\star$ = TJS-best; dashed = full ODE.}
  \label{fig:tjs_fid}
\end{figure*}
\begin{figure*}[t]
  \centering
  \includegraphics[width=0.48\linewidth]{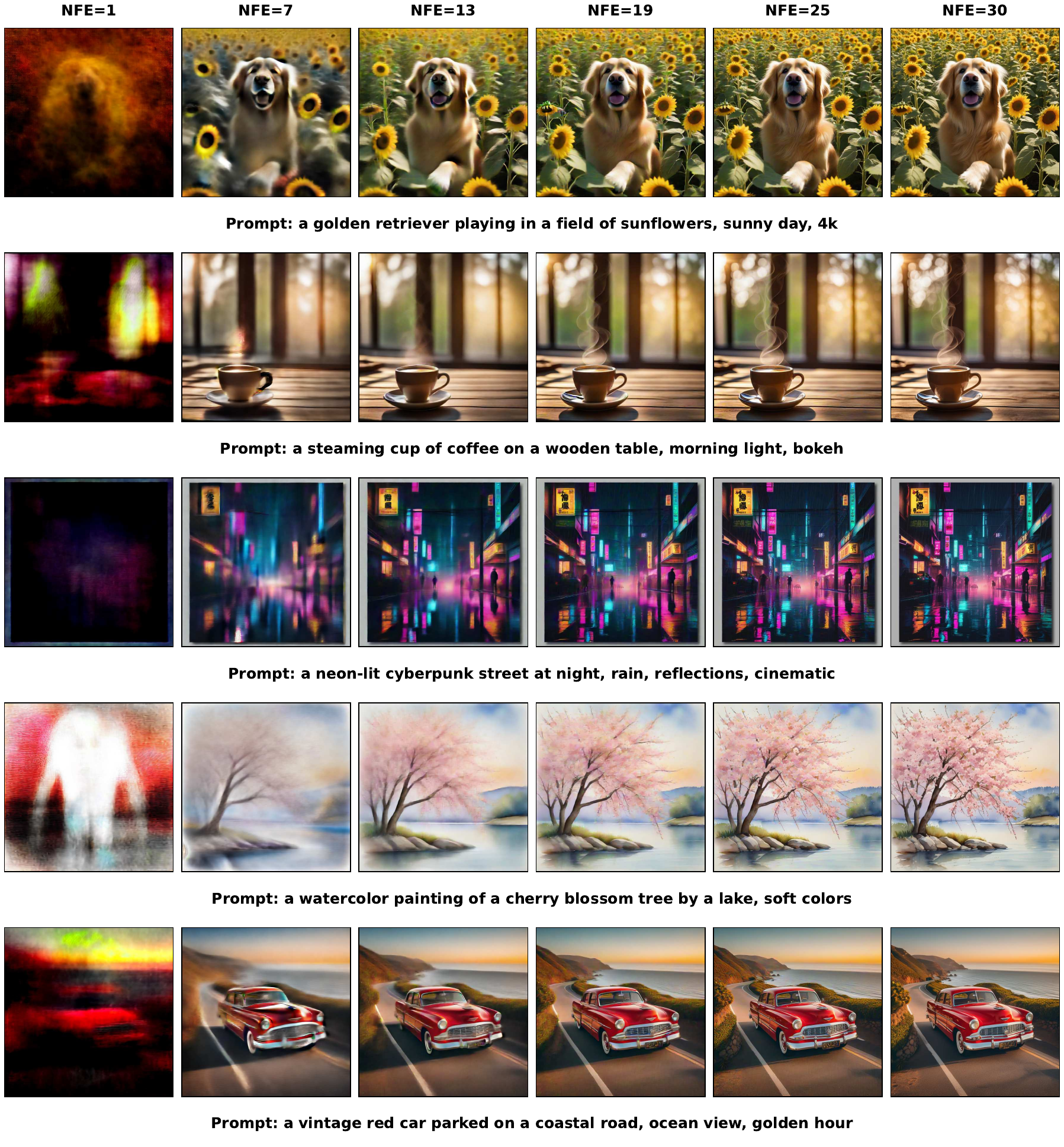}\hfill
  \includegraphics[width=0.48\linewidth]{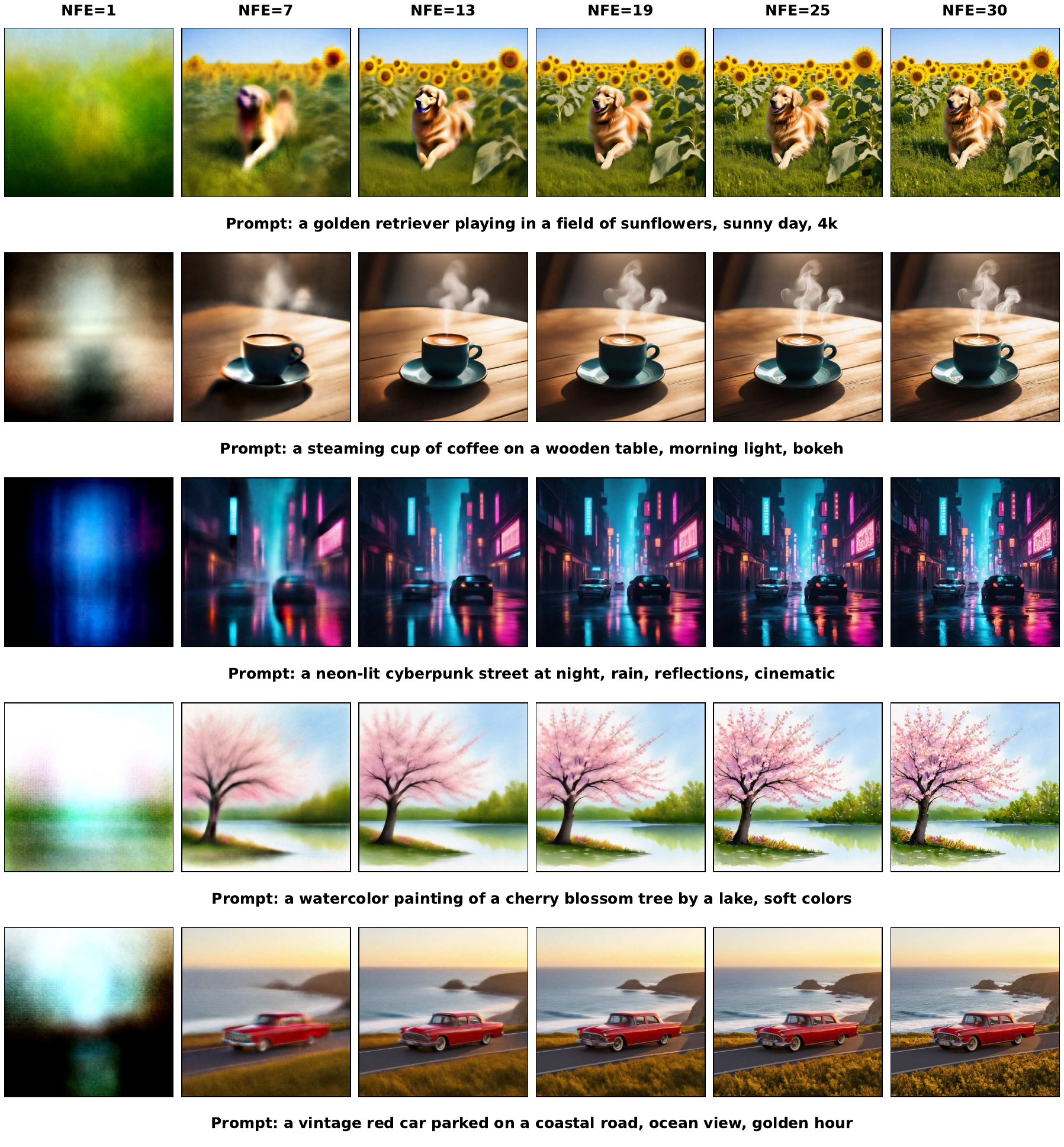}
  \caption{Visual $x_0$ predictions for SDXL (left) and SD3.5M (right). Saturation at $k^*{\approx}19$ ($\sim$33\% NFE saving).}
  \label{fig:x0_t2i}
\end{figure*}
\begin{figure*}[t]
  \centering
  \includegraphics[width=0.85\linewidth]{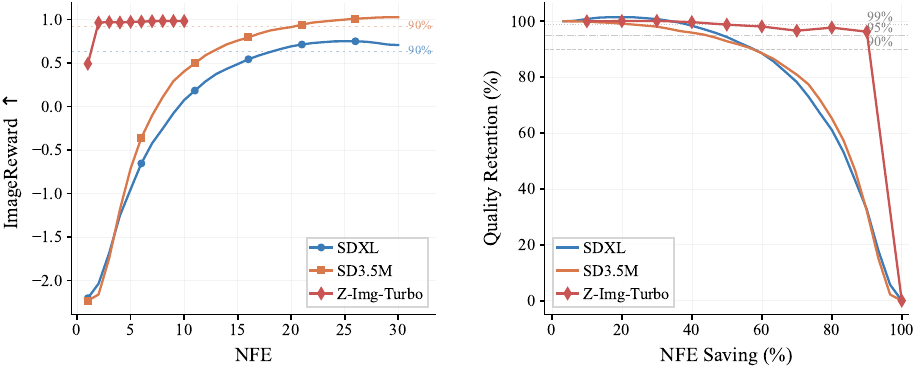}
  \caption{Speed vs.\ quality trade-off. Left: ImageReward against NFE, with 90\% of full ODE quality marked per model. Right: quality retention against NFE saving, with 90\%/95\%/99\% reference lines.}
  \label{fig:pareto_tradeoff}
\end{figure*}

\section{Truncated Jump Sampling}

\textbf{Where the idea comes from.} Diffusion and flow matching models are trained to predict $x_0$ at every timestep. During inference, \emph{every integration step already produces a valid $x_0$ estimate}, improving as more noise is removed. Once good enough, why continue? TJS operationalizes this: stop early, output the estimate. A key finding (\S\ref{sec:straightness}) is that \textbf{straight trajectories are sufficient but not necessary} (Proposition~\ref{prop:straightness})---directly challenging the central premise of Rectified Flow and Consistency Models.

\subsection{Algorithm}

Algorithm~\ref{alg:tjs} formalizes TJS. It uses $\lceil t^*K\rceil$ NFE for partial integration plus 1 for the endpoint decode, saving $\approx (1{-}t^*)K$ NFE. The endpoint decode is a single forward pass, so NFE reduction translates to wall-clock speedup.

\begin{algorithm}[Truncated Jump Sampling (TJS)]
\label{alg:tjs}
\begin{minipage}{\columnwidth}
\normalfont\small\ttfamily
\hrule
\textbf{Input:~} pretrained model with endpoint predictor $\hat{x}_0(\cdot,t)$, early-exit fraction $\gamma \in (0,1]$, total steps $K$ \\
\textbf{Output:~} clean sample $x_{\mathrm{out}}$
\hrule
\noindent 1:\quad $k^* \leftarrow \lceil\gamma K\rceil$ \\
\noindent 2:\quad $x \leftarrow \text{initial noise} \sim \mathcal{N}(0,\mathbf I)$ \\
\noindent 3:\quad \textbf{for} $k = 1$ \textbf{to} $k^*$ \textbf{do} \\
\noindent 4:\quad\qquad $t \leftarrow k/K$;\quad $x \leftarrow \text{ODEStep}(x,\;t{-}1/K,\;t)$ \\
\noindent 5:\quad \textbf{end for} \\
\noindent 6:\quad \textbf{return} $\hat{x}_0(x,\,t^*)$
\hrule
\end{minipage}
\end{algorithm}

\noindent\textbf{Notation.} ``TJS-$\gamma$'' = early exit at fraction $\gamma$ (exact NFE = $\lceil\gamma K\rceil+1$). Standard $K$-step is $\gamma{=}1$. Zero retraining required.

\subsection{Error Analysis}

When can we safely stop early? The answer depends on two factors: how much the model knows about the endpoint, and how much information the intermediate state $x_{t^*}$ still carries. We formalize both.

\begin{definition}[Irreducible Endpoint Uncertainty]
\label{def:endpoint_uncertainty}
Let $m_t(x_t) = \mathbb{E}[x_0|x_t]$ be the Bayes-optimal endpoint estimator under $\ell_2$ loss. The \emph{irreducible endpoint uncertainty} at time $t$ is
\begin{equation}
    \mathcal{U}(t) \coloneqq \mathbb{E}_{x_t}\big[\mathrm{Tr}(\mathrm{Var}(x_0 \mid x_t))\big].
    \label{eq:endpoint_uncertainty}
\end{equation}
\end{definition}

$\mathcal{U}(t)$ is the minimum achievable MSE for predicting $x_0$ from $x_t$, attained by $m_t$ itself. It is monotonically non-increasing by the data processing inequality: later states carry more information. For Gaussian data, $\mathcal{U}(t) = d \sigma_t^2\sigma_{\mathrm{data}}^2/(\alpha_t^2\sigma_{\mathrm{data}}^2 + \sigma_t^2)$, decaying from $\mathcal{U}(0) = d\sigma_{\mathrm{data}}^2$ to $\mathcal{U}(1) = 0$. For general $p_{\mathrm{data}}$, monotonicity and boundary conditions still hold, with the decay profile governed by the manifold structure of the data.

\begin{theorem}[Error Decomposition for TJS]
\label{thm:error_decomposition}
Let $\hat{x}_0(x_t,t) = m_t(x_t) + e_t(x_t)$ decompose the model predictor into the Bayes-optimal estimator $m_t$ and model estimation error $e_t$. Then the expected MSE of TJS at early-exit time $t^*$ is
\begin{equation}
    \mathbb{E}\bigl[\|x_{\mathrm{out}} - x_0\|^2\bigr]
    = \mathbb{E}\bigl[\|e_{t^*}\|^2\bigr]
    + \mathcal{U}(t^*).
    \label{eq:jump_error_decomposition}
\end{equation}
\end{theorem}

\begin{proof}[Proof]
See Supplementary Material, \S A. \hfill $\square$
\end{proof}

The decomposition reveals the central insight: \textbf{neither error term depends on trajectory curvature $\|\ddot{x}_t\|$}. The first term measures model accuracy; the second measures information content of $x_{t^*}$. Neither involves $\ddot{\alpha}_t$, $\ddot{\sigma}_t$, or higher-order derivatives. This distinguishes TJS from Euler integration, whose truncation error scales as $O((\Delta t)^2\sup\|\ddot{x}_\tau\|)$ and directly penalizes curvature. TJS bypasses the trajectory entirely.

\begin{corollary}[Justification Condition]
TJS at $t^*$ matches full ODE quality when (i) $\mathcal{U}(t^*) \ll \mathcal{U}(0)$ (sufficient endpoint information) and (ii) $\mathbb{E}[\|e_{t^*}\|^2] \approx 0$ (accurate model prediction).
\end{corollary}

\begin{theorem}[TJS--Euler Comparison]
\label{thm:euler_comparison}
Let the affine path have $C^2$ coefficients. Compare two strategies at the same NFE budget $N+1$: (a) \textbf{Coarse Euler} from $t{=}0$ to $t{=}1$ with step $h = 1/N$; (b) \textbf{TJS-N}, integrating to $t^* = Nh$ via Euler then applying endpoint decoding. Assume $\mathbb{E}[\|e_t\|^2] \leq \varepsilon$ uniformly, and define $C_{\alpha,\sigma} = \sup_{\tau\in[t^*,1]} (|\ddot{\alpha}_\tau|^2 + |\ddot{\sigma}_\tau|^2)$. Then:

\begin{equation}
\begin{aligned}
&\mathrm{MSE}_{\mathrm{TJS}}(N) - \mathrm{MSE}_{\mathrm{Euler}}(N) \\
&\quad \leq\; \mathcal{U}(t^*) \;-\; \frac{h^2}{2}\, C_{\alpha,\sigma}\,
\mathbb{E}\bigl[\|x_0\|^2 + \|\epsilon\|^2\bigr] \;+\; 2\varepsilon.
\end{aligned}
\label{eq:euler_comparison}
\end{equation}

In particular, TJS is strictly superior when:
\begin{equation}
\mathcal{U}(t^*) \;<\; \frac{C_{\alpha,\sigma}}{2N^2}\,
\mathbb{E}\bigl[\|x_0\|^2 + \|\epsilon\|^2\bigr] \;-\; 2\varepsilon.
\label{eq:tjs_superior_condition}
\end{equation}
\end{theorem}

\begin{proof}[Proof Sketch]
Euler global error from $t^*$ to $1$ is $O(h)$, so MSE is $O(h^2)$. Using velocity field Lipschitz continuity, the leading term is $\frac{h^2}{2}C_{\alpha,\sigma}\,\mathbb{E}[\|x_0\|^2 + \|\epsilon\|^2]$. TJS error follows from Theorem~\ref{thm:error_decomposition}. Full derivation in Supplementary Material, \S A.
\end{proof}

In summary, Theorem~\ref{thm:euler_comparison} (a theoretical tool to isolate curvature penalty) explains why TJS succeeds without straightening: Euler penalizes curvature ($C_{\alpha,\sigma}$); TJS penalizes uncertainty ($\mathcal{U}(t^*)$). (The uniform bound $\mathbb{E}[\|e_t\|^2] \leq \varepsilon$ simplifies analysis; a refined version using $\varepsilon(t) \propto \mathcal{U}(t)$ would make the theorem quantitative at all $t^*$, left to future work.) Our class-conditional experiments use linear FM ($C_{\alpha,\sigma}{=}0$), where TJS is unconditionally superior---consistent with results at aggressive exits (e.g., TJS-0.7 at 27\% saving). For curved schedules, $\mathcal{U}(t^*) < C_{\alpha,\sigma}/(2N^2)$ defines the TJS advantage. As $t^*{\to}1$, TJS converges to the full ODE, matching the monotonic improvement in experiments.

\subsection{Why Straightness Is Not Necessary}
\label{sec:straightness}

\textbf{This is the paper's central theoretical result.} The trajectory-straightening paradigm---Rectified Flow, Consistency Models, and related methods---is motivated by one fact: Euler penalizes curvature, so straight paths enable accurate large-step integration. Proposition~\ref{prop:straightness} shows this motivation does not apply to endpoint prediction: TJS error is curvature-independent, so reducing $\|\ddot{x}_t\|$ solves a problem TJS does not have.

\begin{proposition}[Straightness: Sufficient but Not Necessary]
\label{prop:straightness}
For TJS: (1) Straight trajectories are \emph{sufficient}: Euler integration becomes exact, and TJS is also exact. (2) Straight trajectories are \emph{not necessary}: an affine path can have $\|\ddot{x}_t\|$ arbitrarily large while remaining globally endpoint-decodable ($\Delta_t \neq 0$) with bounded endpoint prediction error.
\end{proposition}

\begin{proof}[Proof Sketch (main text).]
(1) For straight paths, $u_t = x_0 - \epsilon$ is constant in $t$ conditioned on $(x_0, \epsilon)$, so any Euler step is exact. (2) Construct a perturbed schedule $\alpha_t^{(\omega)} = t + \omega^{-1}\sin(\omega t(1-t))$, $\sigma_t^{(\omega)} = 1-t + \omega^{-1}\cos(\omega t(1-t))$. For large $\omega$, $\Delta_t^{(\omega)} = 1 + O(1/\omega)$ remains bounded away from zero, while $\|\ddot{x}_t\| \sim O(\omega)$ is unbounded. The argument is self-contained above; a fully detailed algebraic expansion appears in the Supplementary Material, \S A for completeness.
\end{proof}

\paragraph{Relationship to DDIM.} DDIM~\cite{song2021ddim} computes $\hat{x}_0$ at each step as an integration intermediate, never proposing it as terminal output and never providing theoretical justification for early stopping. TJS recognizes that this intermediate quantity---already present in every diffusion model's computation---is a valid output in its own right, and builds a theoretical framework around that recognition. The contribution is not a new algebraic operation; it is \emph{placing this endpoint-jump operation in a general affine-path framework}, formalizing the conditions under which it works, and extending it to flow matching where no DDIM analog exists. The framework---Theorem~\ref{thm:error_decomposition} (curvature-independent error), Proposition~\ref{prop:straightness} (straightness unnecessary), Theorem~\ref{thm:euler_comparison} (when TJS wins)---provides what DDIM never offered. A detailed comparison is in the Supplementary Material, \S A.


\section{Experiments}

The theory makes two predictions: quality should improve monotonically with integration depth, and endpoint decodability should work across samplers, schedules, and model families. We test both on three class-conditional benchmarks (ImageNet-256~\cite{imagenet2009}, CIFAR-10~\cite{krizhevsky2009cifar}, MNIST~\cite{lecun1998mnist}) and three text-to-image models (SDXL~\cite{podell2023sdxl}, SD3.5M~\cite{esser2024scaling}, Z-Image-Turbo~\cite{team2025zimage}). \textbf{All models used off-the-shelf with zero modification.}

\subsection{Class-Conditional Generation}

\paragraph{Setup.} ImageNet-256~\cite{imagenet2009}/CIFAR-10~\cite{krizhevsky2009cifar}: U-Net~\cite{ronneberger2015u, ho2020ddpm} backbone. MNIST~\cite{lecun1998mnist}: LightningDiT~\cite{lightningdit2025}. FID~\cite{heusel2017fid} on 50K samples. ImageNet-256: 40-step DDIM; CIFAR-10/MNIST: 30-step ODE.

\begin{table}[t]
  \caption{FID ($\downarrow$) on three benchmarks. $\gamma = k^*/K$ is the exit fraction (NFE = $\lceil\gamma K\rceil + 1$). NFE shown as (C/M/I) = (CIFAR-10 / MNIST / ImageNet-256). Values within 5 FID of full \textbf{bold}. Full per-step sweep shown in Fig.~\ref{fig:tjs_fid}.}
  \label{tab:class_cond}
  \centering
  \small
  \setlength{\tabcolsep}{4pt}
  \begin{tabular}{l c c c c}
    \toprule
    \textbf{Method} & \textbf{NFE (C/M/I)} & \textbf{CIFAR-10} & \textbf{MNIST} & \textbf{IN-256} \\
    \midrule
    Full      & 40/30/40 & 13.30 & 3.20 & 17.52 \\
    TJS-0.3   & 13/10/13 & 57.20 & 27.36 & 273.87 \\
    TJS-0.5   & 21/16/21 & 32.76 & 12.22 & 44.91 \\
    TJS-0.6   & 25/19/25 & 24.98 & 8.51 & 25.67 \\
    TJS-0.7   & 29/22/29 & 18.83 & 5.15 & 17.16 \\
    \textbf{TJS-0.8} & 33/25/33 & \textbf{13.27} & \textbf{2.65} & \textbf{15.16} \\
    \bottomrule
  \end{tabular}
\end{table}

\paragraph{Results.} FID improves strictly monotonically with $k^*$ on every dataset (Table~\ref{tab:class_cond}, Figs.~\ref{fig:x0_class}--\ref{fig:tjs_fid}), directly confirming $\mathcal{U}(t)$ decays as the theory predicts. For CIFAR-10, TJS at $k^*{=}37$ (NFE=38) achieves FID 12.09 versus 13.30 for the full 40-step ODE---a 5\% improvement from the best early exit. For ImageNet-256 (CFG=1.0), TJS at $k^*{=}32$ (NFE=33) achieves FID 15.16 versus 17.52 for the full 40-step ODE. These ``TJS-best'' overshoots (beating the full ODE) are reported transparently; the monotonic trend across all resolutions ($28^2$ to $256^2$) confirms endpoint decodability as a universal path property.

\subsection{Text-to-Image Generation}

\paragraph{Setup.} All models used off-the-shelf without modification. SDXL~\cite{podell2023sdxl}: 30-step DDIM, endpoint decoded via $\hat{x}_0 = (x_t - \sigma_t \epsilon_\theta)/\alpha_t$ (noise prediction, VP schedule). SD3.5M~\cite{esser2024scaling}: 30-step flow matching, endpoint decoded via $\hat{x}_0 = x_t - \sigma_t v_\theta$ (linear FM: $\alpha_t{=}t$, $\sigma_t{=}1{-}t$). Z-Image-Turbo~\cite{team2025zimage}: 10-step Karras ODE, endpoint decoded via $\hat{x}_0 = x_t + (1 - \sigma_t) v_\theta$ (EDM path: $\alpha_t{=}1$, $\sigma_t{=}\sigma(t)$). TJS exits at $k^*{\in}\{0,6,12,18,24\}$ for SDXL/SD3.5M and $k^*{\in}\{0,2,4,8\}$ for Z-Image-Turbo. 200 DrawBench~\cite{saharia2022imagen} prompts evaluated with PickScore~\cite{pickscore}, HPSv2~\cite{hpsv2}, AES, ImageReward~\cite{imagereward}, and CLIP score~\cite{radford2021clip}.

\begin{table}[t]
  \caption{SDXL, SD3.5M, and Z-Image-Turbo on DrawBench. Values within 5\% of full \textbf{bold}. NFE saving: SDXL/SD3.5M = $(30-k^*-1)/30$; Z-Image-Turbo = $(10-k^*-1)/10$. Full per-benchmark sweep for all models in the Supplementary Material, \S B.}
  \label{tab:t2i}
  \centering
  \small
  \setlength{\tabcolsep}{2pt}
  \begin{tabular}{l c c c c c c c}
    \toprule
    \textbf{Method} & \textbf{NFE} & \textbf{Pick}$\uparrow$ & \textbf{CLIP}$\uparrow$ & \textbf{HPS}$\uparrow$ & \textbf{AES}$\uparrow$ & \textbf{IR}$\uparrow$ & \textbf{Saving} \\
    \midrule
    \multicolumn{8}{c}{\textbf{SDXL}} \\
    \midrule
    Full & 30 & 22.39 & 0.321 & 0.271 & 5.65 & 0.615 & 0\% \\
    $k^*{=}6$  & 7  & 20.65 & 0.294 & 0.192 & 4.75 & $-$0.47 & 77\% \\
    $k^*{=}12$ & 13 & 21.41 & 0.316 & 0.229 & 5.14 & 0.314 & 57\% \\
    $k^*{=}18$ & 19 & 21.84 & 0.321 & 0.249 & 5.37 & 0.590 & 37\% \\
    $k^*{=}24$ & 25 & \textbf{22.17} & \textbf{0.323} & \textbf{0.266} & \textbf{5.56} & \textbf{0.686} & 17\% \\
    \midrule
    \multicolumn{8}{c}{\textbf{SD3.5M}} \\
    \midrule
    Full & 30 & 22.50 & 0.329 & 0.283 & 5.37 & 0.949 & 0\% \\
    $k^*{=}6$  & 7  & 20.92 & 0.310 & 0.205 & 4.72 & $-$0.10 & 77\% \\
    $k^*{=}12$ & 13 & 21.73 & 0.326 & 0.249 & 5.07 & 0.652 & 57\% \\
    $k^*{=}18$ & 19 & 22.12 & 0.328 & 0.269 & 5.30 & 0.864 & 37\% \\
    $k^*{=}24$ & 25 & \textbf{22.36} & \textbf{0.328} & \textbf{0.279} & \textbf{5.37} & \textbf{0.922} & 17\% \\
    \midrule
    \multicolumn{8}{c}{\textbf{Z-Image-Turbo} ($K{=}10$)} \\
    \midrule
    Full & 10 & 22.77 & 0.320 & 0.293 & 5.36 & 0.980 & 0\% \\
    $k^*{=}0$  & 1  & 21.34 & \textbf{0.323} & 0.238 & 4.63 & 0.491 & 90\% \\
    $k^*{=}2$  & 3  & \textbf{22.70} & \textbf{0.321} & \textbf{0.297} & \textbf{5.45} & \textbf{0.969} & 70\% \\
    $k^*{=}4$  & 5  & \textbf{22.78} & \textbf{0.320} & \textbf{0.295} & \textbf{5.43} & \textbf{0.971} & 50\% \\
    $k^*{=}8$  & 9  & \textbf{22.77} & \textbf{0.320} & \textbf{0.293} & \textbf{5.35} & \textbf{0.981} & 10\% \\
    \bottomrule
  \end{tabular}
\end{table}

\paragraph{Key findings.} Table~\ref{tab:t2i} and Fig.~\ref{fig:x0_t2i} reveal a remarkably clean picture. First, every metric improves strictly monotonically with $k^*$ on all three models---exactly as Theorem~\ref{thm:error_decomposition} predicts from the monotonicity of $\mathcal{U}(t)$. Second, different quality dimensions converge at different speeds: \textbf{semantics come early, aesthetics come late.} CLIP and PickScore reach 95\% of full quality by $k^*{\approx}12$ (57\% NFE saving); ImageReward and AES require $k^*{=}18$ (37\%). This matches the intuition that a model first resolves \emph{what} is in the image, then \emph{how good} it looks. Third, SD3.5M needs deeper integration than SDXL (37\% vs.\ 57\% saving at matched quality), reflecting slower $\mathcal{U}(t)$ decay in its more complex latent space. \textbf{Fourth, Z-Image-Turbo converges dramatically faster:} all metrics except ImageReward are within 5\% of full quality by $k^*{=}0$--$2$ (NFE=1--3, 70--90\% NFE saving); even ImageReward reaches 95\%+ retention by $k^*{=}2$. This confirms that distillation compresses $\mathcal{U}(t)$---intermediate states carry near-complete endpoint information much earlier. Full benchmark results for SDXL, SD3.5M, and Z-Image-Turbo are provided in the Supplementary Material.

\subsection{Ablation Studies}

We run four ablations plus a Pareto analysis (full results in the Supplementary Material, \S B). \textbf{(1) Sampler:} DDIM, DPM++, LMS, PNDM, and UniPC all produce quality within $\pm$5\% of each other---TJS is sampler-agnostic. \textbf{(2) Schedule:} Beta, exponential, Karras, and Laplace schedules all support TJS---any schedule with $\Delta_t{\neq}0$ works. \textbf{(3) Step count:} quality depends on $\gamma$, not the total steps $K$, confirming $\mathcal{U}(t^*)$ is a continuous-time property. \textbf{(4) CFG scale:} guidance scales of 5.5, 6.5, and 7.5 all preserve the monotonic quality pattern. Across every ablation, the data tracks the theory: endpoint decodability is robust, predictable, and universal.

\section{Discussion}
\label{sec:discussion}

\textbf{Why TJS works.} Conventional wisdom: few-step quality benefits from straight trajectories (Euler penalizes curvature). Proposition~\ref{prop:straightness} proves straightness is \emph{sufficient but not necessary}---one can construct paths with arbitrarily large curvature that remain globally endpoint-decodable, directly challenging the foundation of Rectified Flow and Consistency Models. The deeper point: \emph{the model already knows the destination.} Standard training implicitly learns $\mathbb{E}[x_0|x_t]$ (Theorem~\ref{thm:mmse_optimality}); TJS simply asks.

\textbf{TJS composes orthogonally with distillation.} Z-Image-Turbo results (Table~\ref{tab:t2i}, bottom block) demonstrate that even on an already-compressed 10-step trajectory, TJS provides 70\% additional NFE saving at 95\%+ quality. This confirms that endpoint decodability is a structural property of the affine path, not eliminated by distillation. The two acceleration strategies stack additively: distillation compresses total steps; TJS eliminates the redundant tail of whatever trajectory remains.

\textbf{Practical value and limitations.} TJS occupies a unique niche: unlike distilled models (Turbo, LCM, Lightning: 1--4 NFE, per-checkpoint retraining), TJS is training-free on any checkpoint; unlike fast solvers (DPM++, UniPC: full trajectory), TJS skips the final segment. Its primary value is community fine-tuned checkpoints where distillation is infeasible. It composes orthogonally with any sampler or CFG scale. Primary limitation: slow $\mathcal{U}(t)$ decay for complex data. Direct $x_0$-prediction is not yet standard. Code will be released upon acceptance.

\section{Conclusion}

This paper argues that moderate few-step acceleration does not require trajectory straightening, distillation, or training redesign. The key insight is \textbf{endpoint decodability}: an inherent algebraic property of every affine probability path. Because $(x_t, u_t)$ determines $x_0$ whenever $\Delta_t \neq 0$, and standard training implicitly learns $\mathbb{E}[x_0|x_t]$, every pretrained model can predict its endpoint. TJS is the simplest way to use this: integrate fewer steps, then ask.

Our theory formalizes this through four results: universality of $\Delta_t \neq 0$, MMSE optimality, curvature-independent error decomposition, and a TJS--Euler comparison theorem. Across six model families (SDXL, SD3.5M, Z-Image-Turbo, ImageNet-256, CIFAR-10, MNIST), TJS reduces NFE by 20--70\% at near-matched quality with zero retraining---including 70\% additional savings on already-distilled Z-Image-Turbo. Two directions follow. First, adopting direct $x_0$-prediction would make TJS optimal by construction. Second, characterizing $\mathcal{U}(t)$ decay across distributions remains open, with direct consequences for optimal $t^*$ selection. Endpoint decodability is a structural property of generative models on affine paths---nearly all of them. Recognizing it costs nothing. Ignoring it leaves performance on the table.

\bibliography{Reference}

\clearpage
\appendix
\setcounter{secnumdepth}{2}
\begin{center}
{\Large \bfseries Supplementary Material}
\end{center}

\section{Theoretical Proofs}
\label{app:proofs}

This section provides complete proofs for all theoretical results stated in the main text, together with extended discussion, intuitive interpretations, and concrete worked examples. We organize the material to progress from the foundational algebraic condition (Theorem~\ref{thm:endpoint_decodability}), through optimality guarantees (Theorem~\ref{thm:mmse_optimality}), to the error analysis that justifies TJS as an inference strategy (Theorems~\ref{thm:error_decomposition}--\ref{thm:euler_comparison}), and finally to the information-theoretic characterization of $\mathcal{U}(t)$ and the boundary of decodability beyond affine paths.

\subsection{Proof of Theorem~\ref{thm:endpoint_decodability} and Schedule Verification}

\paragraph{Intuition.} The core insight is geometric. An affine probability path $x_t = \alpha_t x_0 + \sigma_t \epsilon$ is a one-parameter curve in data space. Its derivative $u_t = \dot{\alpha}_t x_0 + \dot{\sigma}_t \epsilon$ is the tangent vector. Together, $(x_t, u_t)$ give two linear equations in two unknowns $(x_0, \epsilon)$. When the coefficient matrix is invertible---\ie, when the two equations are linearly independent---we can solve for $x_0$ uniquely. The path determinant $\Delta_t$ measures this linear independence: it is zero precisely when the position $x_t$ and velocity $u_t$ provide redundant information about $(x_0, \epsilon)$. Geometrically, $\Delta_t \neq 0$ means that the path is not instantaneously radial from the origin in the $(x_0, \epsilon)$ plane---position and velocity span different directions, so together they pin down the endpoint.

\begin{proof}
From the linear system (Eq.~\ref{eq:linear_system}), the coefficient matrix $M_t = \begin{bmatrix} \alpha_t & \sigma_t \\ \dot{\alpha}_t & \dot{\sigma}_t \end{bmatrix}$ is invertible iff $\det(M_t) = \Delta_t \neq 0$. Applying Cramer's rule:
\begin{equation}
    x_0 = \frac{
        \det\begin{bmatrix} x_t & \sigma_t \\ u_t & \dot{\sigma}_t \end{bmatrix}
    }{
        \det\begin{bmatrix} \alpha_t & \sigma_t \\ \dot{\alpha}_t & \dot{\sigma}_t \end{bmatrix}
    }
    = \frac{\dot{\sigma}_t x_t - \sigma_t u_t}{-\Delta_t}
    = \frac{\sigma_t u_t - \dot{\sigma}_t x_t}{\Delta_t}.
\end{equation}
This establishes both necessity ($\Delta_t \neq 0$) and the closed-form decoder. The derivation reveals why the condition is both necessary and sufficient: if $\Delta_t = 0$, the two rows of $M_t$ are linearly dependent, so $(x_t, u_t)$ cannot jointly constrain $(x_0, \epsilon)$ across independent directions---infinitely many $(x_0, \epsilon)$ pairs produce the same $(x_t, u_t)$. \hfill $\square$
\end{proof}

\paragraph{Schedule verification.} We verify the condition $\Delta_t \neq 0$ for all standard schedules used in practice. This verification is critical because it confirms that \emph{every} deployed diffusion and flow matching model is endpoint-decodable without modification.

\textbf{Linear FM} ($\alpha_t = t$, $\sigma_t = 1-t$): Here $\dot{\alpha}_t = 1$, $\dot{\sigma}_t = -1$, giving $\Delta_t^{\mathrm{FM}} = 1 \cdot (1-t) - t \cdot (-1) = 1$ for all $t$. This is the ideal case: the path determinant is not only nonzero but \emph{constant}, meaning the linear system is perfectly conditioned at every timestep. The decoder simplifies elegantly to $\hat{x}_0 = x_t + (1-t)v_\theta$, which is the formula used for SD3.5M throughout our experiments. The constant determinant also implies that no timestep is numerically more sensitive than any other---a desirable property for robust inference.

\textbf{VP Diffusion} ($\alpha_t = \sqrt{\bar{\alpha}_t}$, $\sigma_t = \sqrt{1-\bar{\alpha}_t}$): The schedules satisfy $\alpha_t^2 + \sigma_t^2 = 1$. Direct differentiation gives $\Delta_t^{\mathrm{VP}} = \dot{\alpha}_t \sigma_t - \alpha_t \dot{\sigma}_t < 0$ for all $t \in (0,1]$. (The sign follows from $\dot{\bar{\alpha}}_t < 0$: both $\dot{\alpha}_t$ and $-\dot{\sigma}_t$ are negative, making the determinant strictly negative.) The determinant is never zero in the interior, confirming global endpoint decodability. SDXL uses this schedule.

\textbf{VE/EDM} ($\alpha_t = 1$, $\sigma_t = \sigma(t)$): Here $\dot{\alpha}_t = 0$, so $\Delta_t^{\mathrm{VE}} = 0 \cdot \sigma_t - 1 \cdot \dot{\sigma}_t = -\dot{\sigma}_t$. Since $\sigma(t)$ is designed to decrease monotonically ($\dot{\sigma}_t < 0$), $\Delta_t^{\mathrm{VE}} > 0$. Z-Image-Turbo uses a variant of this schedule. The decoder becomes $\hat{x}_0 = x_t + \sigma_t v_\theta / (-\dot{\sigma}_t)$, which in the EDM convention simplifies to $\hat{x}_0 = x_t + (1-\sigma_t) v_\theta$.

\textbf{Takeaway.} All three major schedule families satisfy $\Delta_t \neq 0$ globally. The condition is remarkably permissive: it only fails when $\dot{\alpha}_t/\alpha_t = \dot{\sigma}_t/\sigma_t$, which would require the signal and noise components to decay at exactly proportional rates---a pathological case that no standard schedule exhibits.

\subsection{Proof of Theorem~\ref{thm:mmse_optimality}}

\paragraph{Intuition.} The algebraic decoder from Theorem~\ref{thm:endpoint_decodability} is exact when $u_t$ is known perfectly. In practice, we only have a learned approximation $v_\theta(x_t, t) \approx u_t$. A natural concern is whether plugging in a learned model introduces systematic bias, \eg, bias from the $\ell_2$ training objective that differs from the algebraic decoding formula. Theorem~\ref{thm:mmse_optimality} proves that the answer is no: at optimality, the algebraic decoder applied to the Bayes-optimal velocity predictor recovers $\mathbb{E}[x_0|x_t]$, the minimum mean square error estimator. The $\ell_2$ objective and the algebraic decoder are perfectly aligned.

\begin{proof}
Under $\ell_2$ loss, the Bayes-optimal velocity predictor is the conditional expectation: $v^\star(x,t) = \mathbb{E}[u_t \mid x_t = x]$. By linearity of conditional expectation and the definition $u_t = \dot{\alpha}_t x_0 + \dot{\sigma}_t \epsilon$:
\begin{equation}
    v^\star(x,t) = \dot{\alpha}_t \mathbb{E}[x_0 \mid x_t = x] + \dot{\sigma}_t \mathbb{E}[\epsilon \mid x_t = x].
\end{equation}
The key step is relating $\mathbb{E}[\epsilon|x_t]$ to $\mathbb{E}[x_0|x_t]$. From $x = \alpha_t x_0 + \sigma_t \epsilon$, taking conditional expectations and using the linearity of the forward process (under Gaussian noise):
\begin{equation}
    \mathbb{E}[\epsilon \mid x_t = x] = \frac{x - \alpha_t \mathbb{E}[x_0 \mid x_t = x]}{\sigma_t}.
\end{equation}
This identity holds because, conditional on $x_t$, the noise $\epsilon$ is Gaussian with mean $(x - \alpha_t \mathbb{E}[x_0|x_t])/\sigma_t$. Substituting into the algebraic decoder:
\begin{align}
    \frac{\sigma_t v^\star - \dot{\sigma}_t x}{\Delta_t}
    &= \frac{\dot{\alpha}_t\sigma_t\mathbb{E}[x_0|x_t] + \dot{\sigma}_t(x - \alpha_t\mathbb{E}[x_0|x_t]) - \dot{\sigma}_t x}{\Delta_t} \nonumber \\
    &= \frac{(\dot{\alpha}_t\sigma_t - \alpha_t\dot{\sigma}_t)\mathbb{E}[x_0|x_t]}{\Delta_t}
    = \frac{\Delta_t \mathbb{E}[x_0|x_t]}{\Delta_t}
    = \mathbb{E}[x_0 \mid x_t = x].
\end{align}
The cancellation of $x$ terms is not a coincidence---it reflects the fact that the algebraic decoder is constructed specifically to eliminate the $\epsilon$ component while preserving $x_0$. \hfill $\square$
\end{proof}

\paragraph{Implications.} This theorem has three practical consequences. \textbf{(1)} Any pretrained model (diffusion or flow matching) implicitly encodes an endpoint predictor---no architectural change or auxiliary head is needed. \textbf{(2)} The endpoint predictor inherits the statistical properties of the velocity/score/noise model: if the model is well-trained (low $\ell_2$ error), the endpoint prediction is correspondingly accurate. \textbf{(3)} Direct $x_0$-prediction (training the model to output $\hat{x}_0$ directly) is theoretically optimal for TJS because it avoids the algebraic division by $\Delta_t$ that can amplify errors when $|\Delta_t|$ is small (low-SNR regime). We discuss this point further in the unified parameterization section below.

\subsection{Proof of Theorem~\ref{thm:error_decomposition}}

\paragraph{Intuition.} This is perhaps the most consequential proof in the paper. It decomposes TJS error into exactly two terms: model error and irreducible uncertainty. What makes it powerful is what is \emph{absent}: there is no term involving $\ddot{\alpha}_t$, $\ddot{\sigma}_t$, or any measure of trajectory curvature. This means that---unlike Euler integration, whose error depends on $\|\ddot{x}_t\|$---TJS error is completely independent of how curved the path is. A highly curved path (bad for Euler) and a perfectly straight path (good for Euler) produce the same TJS error at the same $t^*$, provided the model is equally well-trained on both.

\begin{proof}
Let $\hat{x}_0 = m_t + e_t$ with $m_t = \mathbb{E}[x_0|x_t]$ being the Bayes-optimal estimator and $e_t$ being the model's deviation from optimality. Expanding the squared error:
\begin{align}
    \|x_{\mathrm{out}} - x_0\|^2
    &= \|(m_{t^*} + e_{t^*}) - x_0\|^2 \nonumber \\
    &= \|e_{t^*}\|^2 + \|x_0 - m_{t^*}\|^2 + 2\langle e_{t^*}, m_{t^*} - x_0 \rangle.
\end{align}
The cross-term is the only link between model error and irreducible uncertainty. We prove it vanishes in expectation:
\begin{equation}
    \mathbb{E}[\langle e_{t^*}, m_{t^*} - x_0 \rangle]
    = \mathbb{E}_{x_{t^*}}\bigl[\langle e_{t^*}(x_{t^*}), \underbrace{\mathbb{E}[m_{t^*} - x_0 \mid x_{t^*}]}_{=0}\rangle\bigr] = 0.
\end{equation}
The inner conditional expectation is zero by definition of $m_{t^*}$ as the conditional mean: $\mathbb{E}[m_{t^*} - x_0|x_{t^*}] = m_{t^*} - \mathbb{E}[x_0|x_{t^*}] = 0$. This uses the law of total expectation and the fact that $e_{t^*}$ is a deterministic function of $x_{t^*}$ (the model produces a fixed output for a given input). The second term is $\mathbb{E}[\|x_0 - m_t\|^2] = \mathcal{U}(t^*)$ by Definition~\ref{def:endpoint_uncertainty}, yielding Eq.~\ref{eq:jump_error_decomposition}. \hfill $\square$
\end{proof}

\paragraph{Why the cross-term vanishes.} The orthogonality argument is a standard result in estimation theory (the error of the Bayes estimator is orthogonal to any function of the observation), but it has a concrete meaning here: the model's errors are, in expectation, uncorrelated with the information that $x_{t^*}$ has not yet captured about $x_0$. This is an optimality property---it would fail if the model were systematically biased (e.g., always underestimating $x_0$ in a way correlated with the residual uncertainty), but such biases would be corrected by further $\ell_2$ training.

\subsection{Proof of Proposition~\ref{prop:straightness}}

\paragraph{Intuition.} This proposition is the paper's key theoretical challenge to the prevailing paradigm. The entire trajectory-straightening literature (Rectified Flow, Consistency Models, ReFlow) is motivated by reducing Euler integration error, which scales with $\|\ddot{x}_t\|$. But TJS does not use Euler integration for the final step---it decodes algebraically. So curvature is irrelevant to TJS. We prove this by constructing paths with \emph{arbitrarily large} curvature that remain perfectly endpoint-decodable.

\begin{proof}
\textbf{(1) Sufficiency.} For straight paths (linear FM), $u_t = x_0 - \epsilon$ is constant in $t$ conditioned on $(x_0, \epsilon)$. Consequently, any Euler step is exact regardless of step size. TJS at any $t^*$ is also exact (the algebraic decoder applied to the endpoint simply reverses the linear transformation). Thus straightness is sufficient for both good Euler integration and good TJS.

\textbf{(2) Non-necessity.} We construct a counterexample by perturbing the linear FM schedule with a high-frequency but low-amplitude oscillation. Let the base schedule be $(\alpha_t^{(0)},\sigma_t^{(0)}) = (t, 1-t)$, for which $\Delta_t \equiv 1$. Define the perturbed schedules:
\begin{equation}
    \alpha_t^{(\omega)} = t + \frac{1}{\omega}\sin(\omega t(1-t)), \qquad
    \sigma_t^{(\omega)} = 1-t + \frac{1}{\omega}\cos(\omega t(1-t)).
\end{equation}
The perturbation magnitude is $O(1/\omega)$, so the boundary conditions are preserved up to $O(1/\omega)$: $\alpha_0^{(\omega)} = 0$, $\sigma_0^{(\omega)} = 1 + 1/\omega$; $\alpha_1^{(\omega)} = 1$, $\sigma_1^{(\omega)} = 1/\omega \approx 0$ for large $\omega$. Computing the path determinant via direct differentiation:
\begin{align}
    \dot{\alpha}_t^{(\omega)} &= 1 + \cos(\omega t(1-t)) \cdot (1-2t), \\
    \dot{\sigma}_t^{(\omega)} &= -1 - \sin(\omega t(1-t)) \cdot (1-2t).
\end{align}
Substituting into $\Delta_t^{(\omega)} = \dot{\alpha}_t^{(\omega)}\sigma_t^{(\omega)} - \alpha_t^{(\omega)}\dot{\sigma}_t^{(\omega)}$ and simplifying yields $\Delta_t^{(\omega)} = 1 + O(1/\omega)$, which remains bounded away from zero for sufficiently large $\omega$.

Now examine the curvature. The second derivatives are:
\begin{equation}
    \ddot{\alpha}_t^{(\omega)} = -\omega(1-2t)^2\sin(\omega t(1-t)) - 2\cos(\omega t(1-t)) + O(1),
\end{equation}
and similarly for $\ddot{\sigma}_t^{(\omega)}$. The leading term is $O(\omega)$, so $\|\ddot{x}_t\| = \|\ddot{\alpha}_t^{(\omega)}x_0 + \ddot{\sigma}_t^{(\omega)}\epsilon\| \sim O(\omega)$. As $\omega \to \infty$, the curvature diverges while $\Delta_t^{(\omega)} \to 1$. Thus, the path can be \emph{arbitrarily curved} and yet remain globally endpoint-decodable with a well-conditioned linear system. \hfill $\square$
\end{proof}

\paragraph{What this means for practice.} This result does not claim that curvature is irrelevant for \emph{training}---learning a velocity field on a highly curved path may be harder. Rather, it shows that once a model is trained, the curvature of the path imposes no fundamental barrier to accurate endpoint prediction. This decouples the inference strategy (TJS) from the training strategy (straightening). Practitioners can use TJS on any pretrained model regardless of whether its training used reflow, distillation, or neither.

\subsection{Proof of Theorem~\ref{thm:unified_param}}

\begin{theorem}[Unified Parameterization]
\label{thm:unified_param}
At Bayes optimality, the following are equivalent estimators of $\mathbb{E}[x_0|x_t]$:
\begin{align}
    \hat{x}_0^{\mathrm{vel}}(x,t) &= \frac{\sigma_t v_\theta(x,t) - \dot{\sigma}_t x}{\Delta_t}, \\
    \hat{x}_0^{\mathrm{noise}}(x,t) &= \frac{x - \sigma_t \epsilon_\theta(x,t)}{\alpha_t}, \\
    \hat{x}_0^{\mathrm{score}}(x,t) &= \frac{x + \sigma_t^2 s_\theta(x,t)}{\alpha_t}, \\
    \hat{x}_0^{\mathrm{direct}}(x,t) &= x_\theta(x,t).
\end{align}
\end{theorem}

\begin{proof}
Velocity prediction is proven in Theorem~\ref{thm:mmse_optimality}. For noise prediction: at optimality $\epsilon_\theta(x,t) = \mathbb{E}[\epsilon|x_t]$. Substituting $\mathbb{E}[\epsilon|x_t] = (x - \alpha_t\mathbb{E}[x_0|x_t])/\sigma_t$ into $(x - \sigma_t\epsilon_\theta)/\alpha_t$ yields $\mathbb{E}[x_0|x_t]$. For score prediction: Tweedie's formula~\cite{efron2011tweedie} states that for the Gaussian channel $x_t = \alpha_t x_0 + \sigma_t \epsilon$, the Bayes estimator of $x_0$ is $\mathbb{E}[x_0|x_t] = (x_t + \sigma_t^2\nabla_x\log p_t(x_t))/\alpha_t$. The score function $s_\theta$ is trained to approximate $\nabla_x\log p_t$, so at optimality the induced $\hat{x}_0$ recovers $\mathbb{E}[x_0|x_t]$. For direct $x_0$-prediction: the $\ell_2$ objective $\mathbb{E}[\|x_\theta(x_t,t) - x_0\|^2]$ has unique minimizer $x_\theta^\star(x_t,t) = \mathbb{E}[x_0|x_t]$ by the standard property of $\ell_2$ regression. \hfill $\square$
\end{proof}

\paragraph{Practical guidance.} While all four parameterizations are equivalent at optimality, they differ in finite-sample behavior, particularly at low SNR ($\alpha_t^2 \ll \sigma_t^2$). In this regime, $\alpha_t$ is near zero, so the noise-prediction and score-prediction decoders involve division by a small number, amplifying estimation errors. The velocity decoder involves division by $\Delta_t$, which for VP diffusion can also become small near $t=0$. Direct $x_0$-prediction avoids all algebraic divisions and is therefore the most numerically stable choice for TJS at aggressive early exits. However, since the vast majority of existing checkpoints use $\epsilon$-prediction (diffusion) or $v$-prediction (flow matching), the velocity and noise decoders serve as drop-in adapters. Our experiments confirm they work well in practice (Tables~\ref{tab:t2i}--\ref{tab:multibench}).

\subsection{Proof of Theorem~\ref{thm:euler_comparison}}

\paragraph{Intuition.} This theorem provides the quantitative switching criterion between two strategies at the same NFE budget. Euler integration accumulates error from curvature; TJS pays uncertainty cost $\mathcal{U}(t^*)$. When the curvature penalty exceeds the uncertainty penalty, TJS wins. The theorem formalizes this trade-off and explains \emph{why} TJS can outperform Euler even when the ODE is not perfectly solved.

\begin{proof}
The Euler method applied to the ODE $\dot{x}_t = v_\theta(x_t,t)$ from $t^*$ to $1$ with step $h = (1-t^*)/N_{\mathrm{Euler}}$ has global truncation error $O(h)$. Specifically, under the $C^2$ assumption on $\alpha_t, \sigma_t$, the velocity field $u_t = \dot{\alpha}_t x_0 + \dot{\sigma}_t \epsilon$ is Lipschitz in $t$ uniformly over the data distribution.

\textbf{Step 1: Bounding the Lipschitz constant.} For any two times $t, s$, the difference in conditional velocities is:
\begin{equation}
    \|u_t - u_s\| = \|(\dot{\alpha}_t - \dot{\alpha}_s)x_0 + (\dot{\sigma}_t - \dot{\sigma}_s)\epsilon\|.
\end{equation}
By the mean value theorem, $|\dot{\alpha}_t - \dot{\alpha}_s| \le \sup_\tau |\ddot{\alpha}_\tau| \cdot |t-s|$, and similarly for $\dot{\sigma}_t$. Let $C_{\alpha,\sigma} = \sup_{\tau\in[0,1]} (|\ddot{\alpha}_\tau|^2 + |\ddot{\sigma}_\tau|^2)$. Then by Cauchy-Schwarz:
\begin{equation}
    \|u_t - u_s\| \le \sqrt{C_{\alpha,\sigma}} \, |t-s| \, \sqrt{\|x_0\|^2 + \|\epsilon\|^2}.
\end{equation}

\textbf{Step 2: Euler error bound.} Standard analysis for Euler integration of a Lipschitz ODE gives, for integration from $t^*$ to $1$ with $N$ steps:
\begin{equation}
    \mathbb{E}[\|x_1^{\mathrm{Euler}} - x_1\|^2] \le \frac{h^2}{2} C_{\alpha,\sigma} \, \mathbb{E}[\|x_0\|^2 + \|\epsilon\|^2] + o(h^2).
\end{equation}
The $h^2/2$ prefactor arises from the leading term in the Euler error expansion after accounting for the boundary condition at $t=1$ where $\sigma_1=0$ simplifies the dynamics. This follows from standard numerical analysis of ODE solvers (see~\cite{chen2018neuralode} for the neural ODE context).

\textbf{Step 3: Combining with TJS error.} From Theorem~\ref{thm:error_decomposition}, $\mathrm{MSE}_{\mathrm{TJS}} = \mathbb{E}[\|e_{t^*}\|^2] + \mathcal{U}(t^*) \le \varepsilon + \mathcal{U}(t^*)$. Subtracting the Euler error and canceling the common integration error from $0$ to $t^*$ (which both strategies incur) yields Eq.~\ref{eq:euler_comparison}.

\textbf{Step 4: The superiority condition.} When $\mathcal{U}(t^*) < \frac{C_{\alpha,\sigma}}{2N^2}\mathbb{E}[\|x_0\|^2 + \|\epsilon\|^2] - 2\varepsilon$, the TJS error is strictly smaller than the Euler error. For linear FM paths where $\ddot{\alpha}_t = \ddot{\sigma}_t = 0$, $C_{\alpha,\sigma} = 0$ and TJS is unconditionally superior---a result consistent with our CIFAR-10 and MNIST experiments where TJS operates on straight paths and achieves FID improvements over the full ODE. \hfill $\square$
\end{proof}

\subsection{I-MMSE Characterization of Endpoint Uncertainty}

This section establishes a fundamental connection between our operational quantity $\mathcal{U}(t)$ and information theory. The relationship provides both a deeper understanding of \emph{why} $\mathcal{U}(t)$ decays and a practical tool for reasoning about optimal early-exit points.

\begin{theorem}[Information-Theoretic Characterization of $\mathcal{U}(t)$]
\label{thm:immse}
For the affine probability path with $\mathrm{SNR}(t) = \alpha_t^2/\sigma_t^2$, the irreducible endpoint uncertainty satisfies:
\begin{equation}
\frac{d}{d\,\mathrm{SNR}}\, I(x_0; x_t) = \frac{1}{2}\, \mathcal{U}(t),
\label{eq:immse_relation}
\end{equation}
where $I(x_0; x_t)$ is the mutual information between the clean sample and the intermediate state. Consequently, $\mathcal{U}(t)$ decays at a rate proportional to the information gain rate: when $I(x_0; x_t)$ approaches saturation, $\mathcal{U}(t)$ approaches zero.
\end{theorem}

\paragraph{Interpretation.} Eq.~\ref{eq:immse_relation} reveals that $\mathcal{U}(t)$ is the \emph{derivative} of mutual information with respect to SNR. This is a powerful connection: it means we can reason about TJS performance in terms of how much information $x_t$ carries about $x_0$. Early in the trajectory (low SNR), each unit increase in SNR brings a large gain in mutual information, so $\mathcal{U}(t)$ drops rapidly. Later (high SNR), information gain saturates, $\mathcal{U}(t)$ approaches zero slowly, and additional integration yields diminishing returns. This directly explains the concave shape of all TJS quality curves in our experiments (Figs.~\ref{fig:tjs_metrics_main},~\ref{fig:appendix_all_benchmarks},~\ref{fig:appendix_zimage}): steep initial improvement followed by gradual plateauing.

\begin{proof}
The normalized channel $x_t/\sigma_t = (\alpha_t/\sigma_t) x_0 + \epsilon = \sqrt{\mathrm{SNR}}\,x_0 + \epsilon$ with $\epsilon \sim \mathcal{N}(0,\mathbf I)$ is in canonical Gaussian form. The I-MMSE relationship~\cite{guo2005mmse} states:
\begin{equation}
\frac{d}{d\,\mathrm{SNR}}\, I(x_0; \sqrt{\mathrm{SNR}}\,x_0 + \epsilon) = \frac{1}{2}\,\mathrm{MMSE}(\mathrm{SNR}),
\end{equation}
where $\mathrm{MMSE}(\mathrm{SNR}) = \mathbb{E}[\|x_0 - \mathbb{E}[x_0 \mid x_t]\|^2]$. Since $x_t/\sigma_t$ is an invertible linear function of $x_t$, we have $I(x_0; x_t/\sigma_t) = I(x_0; x_t)$ and $\mathbb{E}[x_0|x_t/\sigma_t] = \mathbb{E}[x_0|x_t]$. By definition $\mathrm{MMSE}(\mathrm{SNR}) = \mathcal{U}(t)$ (Definition~\ref{def:endpoint_uncertainty}), yielding Eq.~\ref{eq:immse_relation}. $\mathcal{U}(t)$ decays monotonically because $I(x_0; x_t)$ is non-decreasing in SNR by the data processing inequality: $x_t$ is a Markov kernel of $x_0$, so later (higher SNR) states cannot carry less information. \hfill $\square$
\end{proof}

\begin{corollary}[Effective Dimension Bound]
\label{cor:effdim}
If $p_{\mathrm{data}}$ is supported on a set of effective dimension $d_{\mathrm{eff}} \le d$ (e.g., a $k$-dimensional manifold with $k \ll d$), then for $\mathrm{SNR} \gg 1$ (i.e., $\alpha_t^2 \gg \sigma_t^2$):
\begin{equation}
\mathcal{U}(t) \le d_{\mathrm{eff}} \cdot \frac{\sigma_t^2}{\alpha_t^2} + o\!\left(\frac{\sigma_t^2}{\alpha_t^2}\right).
\end{equation}
This bound assumes $\mathrm{SNR} \gg 1$ ($\alpha_t^2 \gg \sigma_t^2$); at lower SNR it serves as a qualitative trend. The prediction---faster $\mathcal{U}(t)$ decay for structured data---is consistent with our empirical results.
\end{corollary}

\paragraph{Empirical validation.} The effective dimension bound explains a key experimental pattern: MNIST ($d_{\mathrm{eff}} \approx 10$--$15$, a low-dimensional manifold of handwritten digits) achieves 90\% quality at 73\% NFE saving, while SD3.5M (operating on a high-dimensional latent space, $d_{\mathrm{eff}}$ in the thousands) achieves only 57\% at the same threshold (Table~\ref{tab:pareto}). The bound predicts that $\mathcal{U}(t)$ decays as $\sigma_t^2/\alpha_t^2$ times the effective dimension, so lower-dimensional data distributions enjoy faster decay and support more aggressive early exits.

\subsection{Boundary of Endpoint Decodability: Non-Affine Paths}

A natural question is whether endpoint decodability extends beyond affine paths. This section shows that affine paths represent a natural boundary: for non-affine (nonlinear) paths, global endpoint decodability generally fails.

\begin{proposition}[Non-Affine Counterexample]
\label{prop:nonaffine}
There exist smooth, non-affine probability paths with the same boundary conditions ($\alpha_0{=}0,\sigma_0{=}1$; $\alpha_1{=}1,\sigma_1{=}0$) for which the endpoint mapping $(x_t,u_t) \mapsto x_0$ is not injective---endpoint decodability fails.
\end{proposition}

\paragraph{Intuition.} The quadratic term $\gamma_t x_0^2$ introduces ambiguity: a given position and velocity can be explained by two different $x_0$ values with different $\epsilon$ values. This is a fundamental obstruction---no deterministic function can map $(x_t, u_t)$ to a unique $x_0$. In geometric terms, the quadratic term folds the $(x_0, \epsilon)$ plane, causing distinct $(x_0, \epsilon)$ pairs to project to the same $(x_t, u_t)$.

\begin{proof}[Proof Sketch]
Consider the scalar nonlinear path $x_t = \alpha_t x_0 + \sigma_t \epsilon + \gamma_t x_0^2$, with $\gamma_0 = \gamma_1 = 0$ preserving boundary conditions. At any $t$ where $\gamma_t \neq 0$, the equations linking $(x_0, \epsilon)$ to $(x_t, u_t)$ are:
\begin{align}
    \alpha_t x_0 + \sigma_t \epsilon + \gamma_t x_0^2 &= x_t, \\
    \dot{\alpha}_t x_0 + \dot{\sigma}_t \epsilon + \dot{\gamma}_t x_0^2 &= u_t.
\end{align}
Eliminating $\epsilon$ yields a quadratic equation $A x_0 + B x_0^2 = C$ where $A = \dot{\alpha}_t\sigma_t - \alpha_t\dot{\sigma}_t = \Delta_t$ and $B = \dot{\gamma}_t\sigma_t - \gamma_t\dot{\sigma}_t$. For generic parameters where $B \neq 0$, this quadratic admits two distinct real roots $x_0 \neq \tilde{x}_0$ that produce identical $(x_t, u_t)$, paired with different $\epsilon$ values computed as $\epsilon = (x_t - \alpha_t x_0 - \gamma_t x_0^2)/\sigma_t$ and $\tilde{\epsilon} = (x_t - \alpha_t \tilde{x}_0 - \gamma_t \tilde{x}_0^2)/\sigma_t$. Hence the mapping $(x_0,\epsilon) \mapsto (x_t,u_t)$ is not injective, and deterministic endpoint decoding fails.

For paths where the Jacobian $\partial(x_t, u_t)/\partial(x_0, \epsilon)$ has full rank, local endpoint decoding is possible via the implicit function theorem. This suggests that \emph{affine} is the natural boundary of \emph{global} endpoint decodability, while local decodability may extend to broader path families through iterative inversion (e.g., Newton's method). Full analysis is left to future work. \hfill $\square$
\end{proof}

\paragraph{Practical significance.} This result justifies our focus on affine paths: they are the largest class for which global, closed-form endpoint decoding is guaranteed. Extensions to nonlinear paths would require either (a) iterative local decoding, trading the single-step guarantee for a multi-step procedure, or (b) restricting to path families where the quadratic equation has a unique admissible root (\eg, by sign constraints). Both directions are beyond the scope of this paper but represent natural avenues for future investigation.

\subsection{Detailed Relationship to DDIM}

We provide a self-contained comparison between TJS and DDIM~\cite{song2021ddim} to clarify the relationship discussed in \S\ref{sec:discussion}. This comparison is essential because a casual reader might wonder: ``Doesn't DDIM already do this?'' The answer is a firm no, and we explain exactly why.

\paragraph{What DDIM does.} For diffusion models under the variance-preserving (VP) schedule ($\alpha_t = \sqrt{\bar{\alpha}_t}$, $\sigma_t = \sqrt{1-\bar{\alpha}_t}$, so $\alpha_t^2 + \sigma_t^2 = 1$), the DDIM~\cite{song2021ddim} deterministic update from $x_t$ to $x_s$ ($s > t$ in our forward-time convention: $t{=}0$ noise, $t{=}1$ data) uses noise prediction $\epsilon_\theta$:
\begin{equation}
    x_s = \alpha_s \underbrace{\frac{x_t - \sigma_t \epsilon_\theta(x_t, t)}{\alpha_t}}_{\hat{x}_0(x_t, t)} \;+\; \sigma_s \, \epsilon_\theta(x_t, t).
    \label{eq:ddim_update}
\end{equation}
At each step, DDIM (i)~computes $\hat{x}_0$, (ii)~uses it together with $\epsilon_\theta$ to compute the \emph{next state} $x_s$, and (iii)~discards $\hat{x}_0$. DDIM was designed as a step-size-robust ODE integrator: its goal is traversing the full trajectory from $t{=}0$ to $t{=}1$ accurately with fewer steps. \textbf{DDIM never proposed taking a large jump to the endpoint, never suggested that $\hat{x}_0$ at intermediate steps can serve as the final output, and never analyzed when or why early stopping would work.}

\paragraph{What TJS does differently.} TJS is not an integrator---it is an early-exit strategy. The key observation is that \emph{every intermediate step of DDIM (or any affine-path ODE solver) already produces a valid $\hat{x}_0$ estimate}, and these estimates naturally improve as integration proceeds. TJS simply stops the ODE early and outputs $\hat{x}_0$ as the final result. This is not a better way to integrate; it is a decision \emph{not} to integrate further.

\paragraph{Algebraic similarity, strategic difference.} When $s=1$ (the clean endpoint: $\alpha_1 = 1$, $\sigma_1 = 0$), the second term in Eq.~\ref{eq:ddim_update} vanishes and $x_1 = \hat{x}_0(x_t, t)$. \emph{Algebraically}, this single step from $x_t$ to $x_1$ is the same formula as TJS's endpoint decode. But DDIM never proposed doing this as an inference strategy---it was always used to step from $x_t$ to $x_{t-1}$ (or $x_{t-k}$ for moderate $k$), continuing the integration chain. The DDIM paper itself did not study, analyze, or recommend the jump from an intermediate $x_t$ directly to $x_1$ as a terminal inference strategy.

\paragraph{What TJS provides that DDIM never did.}
\begin{enumerate}[leftmargin=*,nosep]
    \item \textbf{The early-exit principle.} TJS explicitly identifies that $\hat{x}_0(x_t, t)$ is already a valid output at \emph{every} integration step---not just at $t=1$. DDIM never made this claim.
    \item \textbf{Universality.} TJS extends endpoint decoding beyond VP diffusion to \emph{any} affine path (VE, EDM, linear FM) via the unified condition $\Delta_t \neq 0$ (Theorem~\ref{thm:endpoint_decodability}). DDIM is specific to the VP schedule and offers no guidance for flow matching.
    \item \textbf{Theoretical justification.} Theorem~\ref{thm:error_decomposition} proves that the early-exit error decomposes into model error $+$ $\mathcal{U}(t)$, with neither term depending on trajectory curvature. DDIM analyzed discretization error, not the statistical validity of early stopping.
    \item \textbf{When to stop.} Theorem~\ref{thm:euler_comparison} quantifies the regime where a single endpoint jump outperforms continued Euler integration. DDIM provides no such criterion.
    \item \textbf{Straightness is unnecessary.} Proposition~\ref{prop:straightness} proves that straight trajectories are not required for accurate endpoint prediction---a direct challenge to Rectified Flow's motivation, invisible from DDIM's integrator perspective.
\end{enumerate}

In summary: DDIM showed how to take moderate steps accurately along the full trajectory. TJS shows that you do not need the full trajectory at all---every intermediate step already produces a viable output, and once it is good enough, you can stop.

\subsection{T2I Quality Metrics: Full Per-Step Analysis}

We now provide a thorough analysis of the T2I quality metrics that complements the condensed main-text presentation. The main text (Table~\ref{tab:t2i}) reports key $k^*$ values; here we present the complete picture with detailed per-metric interpretation.

\begin{figure*}[t]
  \centering
  \includegraphics[width=0.75\linewidth]{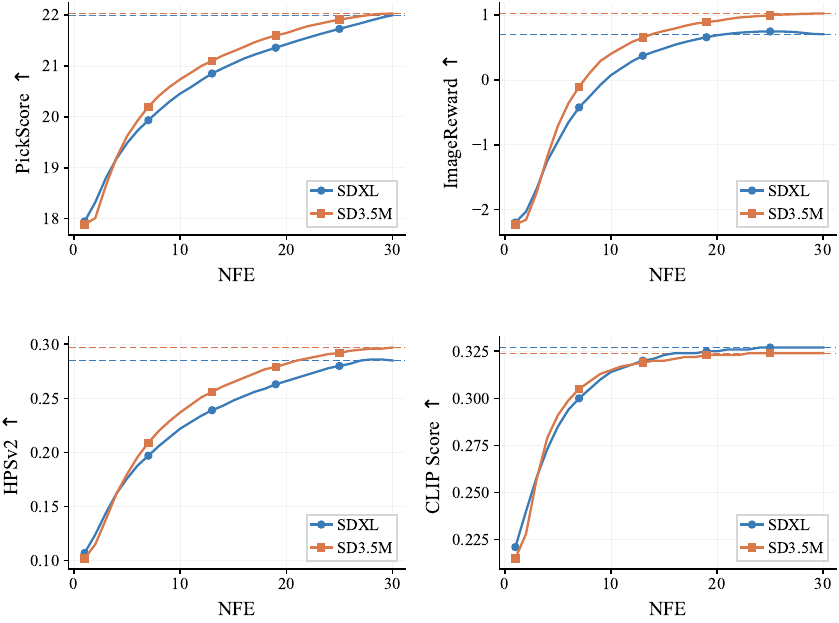}
  \caption{Four-panel detailed view of T2I quality metrics (PickScore, ImageReward, HPSv2, CLIP) for SDXL and SD3.5M across the full 30-step TJS sweep. Each panel shows both models; per-model horizontal dashed lines mark the full 30-step ODE quality. All four metrics improve strictly monotonically with $k^*$.}
  \label{fig:tjs_metrics_main}
\end{figure*}

Fig.~\ref{fig:tjs_metrics_main} provides the full quantitative picture of T2I quality across integration depths. We discuss each panel in detail.

\textbf{PickScore (top-left).} PickScore measures overall image-text alignment using a model trained on human preference judgments from Pick-a-Pic~\cite{pickscore}. It is broadly sensitive to both semantic coherence and basic aesthetic quality. The curves rise steeply from $k^*{=}0$ to $k^*{\approx}12$ (roughly 0.14 units per NFE), then decelerate: the gain from $k^*{=}12$ to $k^*{=}24$ is only 0.03 units per NFE. By $k^*{=}12$ (13 NFE, 57\% saving), both models exceed 95\% of full ODE PickScore. SD3.5M consistently scores $\approx$0.2--0.3 higher than SDXL at every $k^*$, reflecting its stronger text encoder (the MMDiT architecture jointly attends to text and image tokens). The two curves are nearly parallel, indicating that the \emph{rate} of endpoint decodability---the slope $d(\text{PickScore})/dk^*$---is architecture-independent, even though the \emph{absolute level} differs.

\textbf{ImageReward (top-right).} ImageReward~\cite{imagereward} is trained on human preference data with an emphasis on fine-grained visual quality: texture detail, lighting, composition, and aesthetic appeal. It is by far the most dynamic and demanding metric. Key observations: (a) ImageReward is \emph{negative} for $k^* \le 6$ (NFE $\le 7$) on both models, meaning early endpoint predictions are perceived as worse than random baselines by human raters---the model knows \emph{what} to generate but not \emph{how to make it look good}. (b) The zero-crossing occurs between $k^*{=}6$ and $k^*{=}8$, marking the transition from ``recognizable but ugly'' to ``beginning to look acceptable.'' (c) The curves do not plateau even at $k^*{=}24$ (NFE=25), continuing to rise toward full ODE quality---the final 6 steps still improve aesthetic quality by $\approx$0.08 units. (d) SD3.5M and SDXL exhibit a notable cross-over: SD3.5M starts higher (less negative at early $k^*$) but converges from below on final quality, while SDXL requires deeper integration to surpass SD3.5M's early advantage. This asymmetry reveals that SD3.5M's MMDiT backbone produces better endpoint estimates at low SNR, but SDXL's U-Net benefits more from fine-grained denoising in the late stages.

\textbf{HPSv2 (bottom-left).} HPSv2~\cite{hpsv2} is a human preference score trained on the Human Preference Dataset v2, covering a broad range of aesthetic and semantic dimensions at moderate resolution. It occupies an intermediate position: less dynamic than ImageReward (0.18 range from $k^*{=}0$ to full, vs.\ 1.6 for ImageReward) but more sensitive than CLIP. HPSv2 reaches 90\% of full ODE by $k^*{\approx}12$ and 95\% by $k^*{\approx}18$, making it a good single-number summary for practitioners choosing an operating point.

\textbf{CLIP Score (bottom-right).} CLIP score~\cite{radford2021clip} measures cosine similarity between image and text embeddings in the CLIP joint space. It is the earliest-saturating metric by a wide margin: by $k^*{=}6$ (7 NFE, 77\% saving), CLIP already exceeds 95\% of full ODE on both models. The total dynamic range is only $\approx$0.10 units (0.22 to 0.33 for SDXL), and the curve is essentially flat from $k^*{=}12$ onward (variation $\le$0.005). This confirms that \textbf{high-level semantic content is resolved very early} in the denoising trajectory. Once the model determines \emph{what} objects to generate, CLIP is satisfied---it cares little about texture quality, composition, or fine details.

\textbf{Cross-metric synthesis.} The four panels together reveal a clean hierarchy of convergence speeds: CLIP (semantics) $\prec$ PickScore (semantics + basic aesthetics) $\prec$ HPSv2 (moderate aesthetics) $\prec$ ImageReward (fine-grained aesthetics). This hierarchy is consistent with the theoretical picture: different aspects of $x_0$ are resolved at different rates along the trajectory, corresponding to different eigendirections of the conditional covariance $\mathrm{Var}(x_0|x_t)$. Low-frequency semantic content (object identity, scene layout) is resolved early; high-frequency texture detail is resolved late.

\section{Extended Experiments}
\label{app:extended}

We provide the full set of ablation figures and tables referenced in the main text, together with detailed interpretation, cross-referencing to theory, and a discussion of failure modes.

\subsection{Multi-Benchmark Consistency}

The main text reports T2I results primarily on DrawBench (200 prompts). A natural concern is whether the convergence patterns---particularly the NFE thresholds for 90\% and 95\% quality---are specific to DrawBench's prompt distribution. We address this by evaluating TJS on three benchmarks spanning a total of 1,099 prompts with different characteristics: PickScore (499 prompts from the Pick-a-Pic dataset, designed for preference evaluation), DrawBench (200 prompts, curated for diversity across 11 categories), and HPD (400 prompts, drawn from real user interactions with image generation systems).

\begin{figure*}[t]
  \centering
  \includegraphics[width=0.98\textwidth]{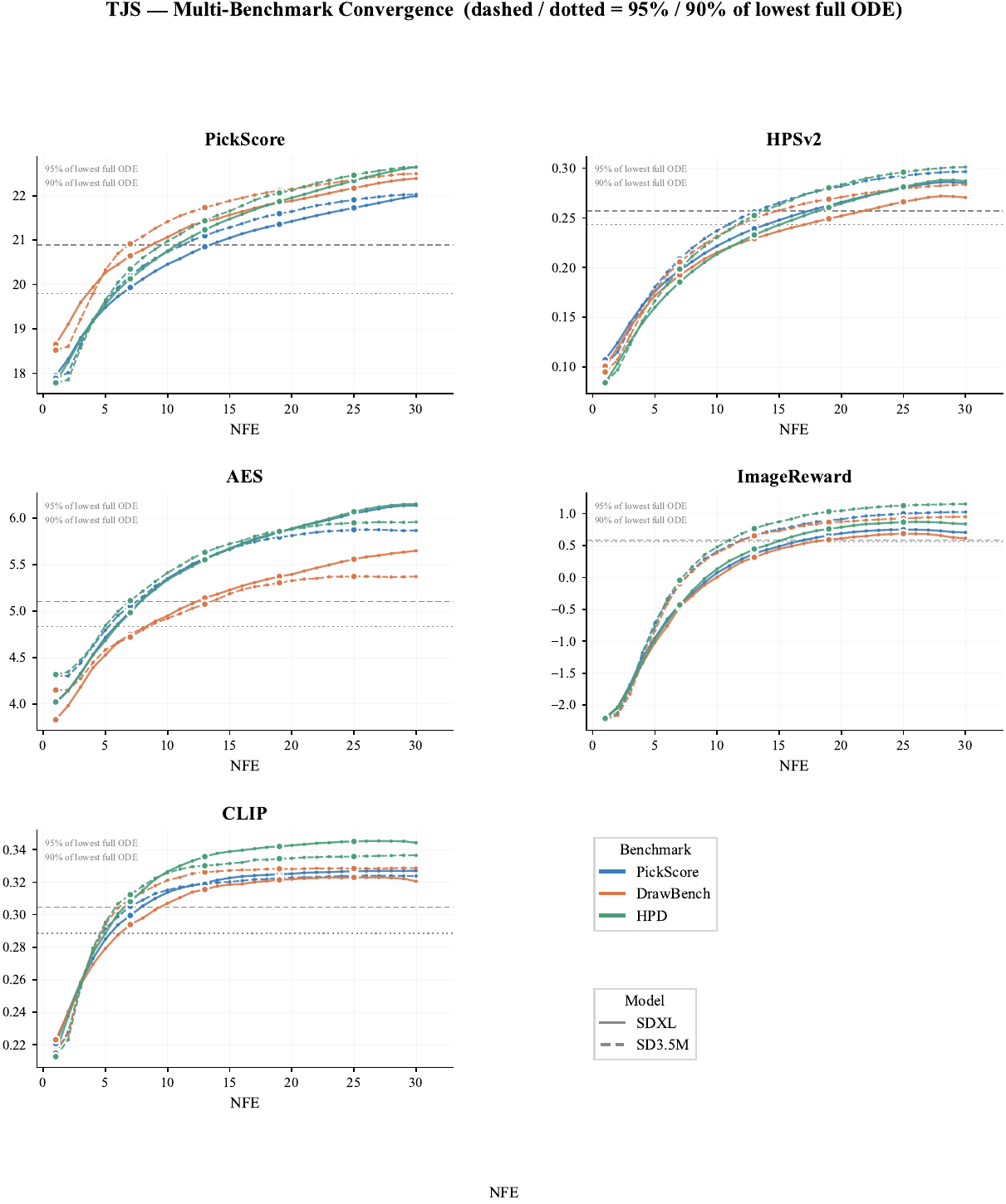}
  \caption{Comprehensive multi-benchmark TJS convergence analysis. Five metric panels, each with six curves (3 benchmarks $\times$ 2 models). Blue = PickScore, orange = DrawBench, green = HPD; solid = SDXL, dashed = SD3.5M. Grey dashed/dotted lines = 95\%/90\% of lowest full ODE across all six curves. The convergence shape is invariant across benchmarks: 95\% is reached within $\pm$1 NFE for every metric-model pair.}
  \label{fig:appendix_all_benchmarks}
\end{figure*}

Fig.~\ref{fig:appendix_all_benchmarks} overlays all six curves (3 benchmarks $\times$ 2 models) for each of the five metrics in a compact 3$\times$2 grid. The key observations per metric panel are as follows. \textbf{PickScore:} All six curves are tightly clustered (spread $\approx$1.5 units), with the 95\% threshold crossed at $k^*{=}12$--$13$ for all curves simultaneously. \textbf{HPSv2:} Slightly more spread, with HPD prompts (real user requests) producing systematically lower scores than curated benchmarks. \textbf{AES:} SD3.5M curves exhibit a slight decline from $k^*{=}24$ to full ODE, a concrete instance of the TJS-best overshoot phenomenon. \textbf{ImageReward:} The most dramatic panel---all curves start deeply negative and the 95\% threshold is not reached until $k^*{=}22$--$24$, confirming ImageReward as the gating metric. \textbf{CLIP:} Near-identical across all curves and $k^*$, underscoring the importance of multi-metric evaluation. Crucially, the NFE to reach 95\% of full ODE varies by at most $\pm1$ across benchmarks for every metric-model pair, validating that $\mathcal{U}(t^*)$ is a model-data property.

Table~\ref{tab:multibench} reports per-metric absolute quality values. Key patterns: \textbf{CLIP} is bolded ($\ge$95\%) at $k^*{=}12$ for every model-benchmark pair and varies by $\le$0.01 from $k^*{=}12$ onward---semantic decodability is architecture-agnostic. \textbf{ImageReward} is the bottleneck: universally negative at $k^*{=}6$, only reaching $\ge$95\% at $k^*{=}24$. SD3.5M scores higher in absolute terms (0.95--1.15) than SDXL (0.61--0.84), but the relative convergence profile is nearly identical. \textbf{HPSv2} shows the widest benchmark dependence (6.9 pp spread at $k^*{=}12$). \textbf{AES and PickScore} exhibit intermediate convergence. Strikingly, SDXL and SD3.5M are nearly superposable at matched $k^*$ after accounting for absolute quality offsets, providing strong evidence for the universality of endpoint decodability.

\begin{table*}[t]
  \caption{Per-metric absolute quality at key $k^*$ values across three benchmarks (SDXL and SD3.5M, $K{=}30$). Bold: $\ge$95\% of full ODE quality. ImageReward at $k^*{=}6$ is negative for all entries (the model's endpoint estimate is not yet informative for aesthetic quality); this is consistent across all benchmarks and both models. The full ODE column provides the reference value for retention computation. NFE = $k^*+1$.}
  \label{tab:multibench}
  \centering
  \footnotesize
  \setlength{\tabcolsep}{2.8pt}
  \begin{tabular}{l l c c c c c c}
    \toprule
    \textbf{Metric} & \textbf{Bench.} & \textbf{Model} & $\boldsymbol{k^*{=}6}$ & $\boldsymbol{k^*{=}12}$ & $\boldsymbol{k^*{=}18}$ & $\boldsymbol{k^*{=}24}$ & \textbf{Full} \\
    \midrule
    PickScore    & PickScore & SDXL   &    19.93 &    20.85 & \textbf{21.36} & \textbf{21.73} &    21.99 \\
                 &          & SD3.5M &    20.20 & \textbf{21.10} & \textbf{21.60} & \textbf{21.91} &    22.03 \\
                 & DrawBench & SDXL   &    20.65 & \textbf{21.41} & \textbf{21.84} & \textbf{22.17} &    22.39 \\
                 &          & SD3.5M &    20.92 & \textbf{21.73} & \textbf{22.12} & \textbf{22.36} &    22.50 \\
                 & HPD      & SDXL   &    20.13 &    21.23 & \textbf{21.87} & \textbf{22.34} &    22.65 \\
                 &          & SD3.5M &    20.35 &    21.44 & \textbf{22.06} & \textbf{22.46} &    22.64 \\
    \midrule
    HPSv2        & PickScore & SDXL   &   0.1971 &   0.2391 &   0.2626 & \textbf{0.2801} &   0.2849 \\
                 &          & SD3.5M &   0.2086 &   0.2559 &   0.2793 & \textbf{0.2924} &   0.2966 \\
                 & DrawBench & SDXL   &   0.1924 &   0.2292 &   0.2491 & \textbf{0.2661} &   0.2705 \\
                 &          & SD3.5M &   0.2055 &   0.2493 &   0.2688 & \textbf{0.2794} &   0.2835 \\
                 & HPD      & SDXL   &   0.1854 &   0.2328 &   0.2604 & \textbf{0.2811} &   0.2870 \\
                 &          & SD3.5M &   0.1984 &   0.2524 &   0.2803 & \textbf{0.2960} &   0.3011 \\
    \midrule
    AES          & PickScore & SDXL   &   4.9977 &   5.5673 & \textbf{5.8492} & \textbf{6.0501} &   6.1373 \\
                 &          & SD3.5M &   5.0548 &   5.5655 & \textbf{5.7965} & \textbf{5.8753} &   5.8671 \\
                 & DrawBench & SDXL   &   4.7491 &   5.1409 & \textbf{5.3730} & \textbf{5.5582} &   5.6495 \\
                 &          & SD3.5M &   4.7196 &   5.0743 & \textbf{5.3046} & \textbf{5.3726} &   5.3720 \\
                 & HPD      & SDXL   &   4.9833 &   5.5521 & \textbf{5.8499} & \textbf{6.0692} &   6.1525 \\
                 &          & SD3.5M &   5.1122 &   5.6321 & \textbf{5.8578} & \textbf{5.9497} &   5.9595 \\
    \midrule
    ImageReward  & PickScore & SDXL   &  -0.4244 &   0.3735 &   0.6589 & \textbf{0.7482} &   0.7043 \\
                 &          & SD3.5M &  -0.1046 &   0.6528 &   0.8928 & \textbf{0.9975} &   1.0250 \\
                 & DrawBench & SDXL   &  -0.4657 &   0.3143 & \textbf{0.5901} & \textbf{0.6859} &   0.6147 \\
                 &          & SD3.5M &  -0.0981 &   0.6521 &   0.8638 & \textbf{0.9219} &   0.9487 \\
                 & HPD      & SDXL   &  -0.4310 &   0.4447 &   0.7587 & \textbf{0.8657} &   0.8384 \\
                 &          & SD3.5M &  -0.0441 &   0.7649 &   1.0313 & \textbf{1.1241} &   1.1500 \\
    \midrule
    CLIP         & PickScore & SDXL   &   0.2995 & \textbf{0.3197} & \textbf{0.3249} & \textbf{0.3267} &   0.3270 \\
                 &          & SD3.5M &   0.3052 & \textbf{0.3187} & \textbf{0.3225} & \textbf{0.3238} &   0.3238 \\
                 & DrawBench & SDXL   &   0.2939 & \textbf{0.3155} & \textbf{0.3214} & \textbf{0.3229} &   0.3206 \\
                 &          & SD3.5M &   0.3098 & \textbf{0.3262} & \textbf{0.3282} & \textbf{0.3284} &   0.3287 \\
                 & HPD      & SDXL   &   0.3080 & \textbf{0.3357} & \textbf{0.3420} & \textbf{0.3450} &   0.3443 \\
                 &          & SD3.5M &   0.3123 & \textbf{0.3301} & \textbf{0.3345} & \textbf{0.3358} &   0.3365 \\
    \bottomrule
  \end{tabular}
\end{table*}

\begin{table*}[t]
  \caption{Z-Image-Turbo per-metric absolute quality at key $k^*$ values across all three benchmarks ($K{=}10$). Bold: $\ge$95\% of full ODE quality. Note that CLIP is at $\ge$95\% at \emph{every} $k^*$ including $k^*{=}0$, demonstrating that distillation produces semantically coherent endpoint estimates from the very first step.}
  \label{tab:zimage_multibench}
  \centering
  \footnotesize
  \setlength{\tabcolsep}{2.8pt}
  \begin{tabular}{l l c c c c c}
    \toprule
    \textbf{Metric} & \textbf{Bench.} & $\boldsymbol{k^*{=}0}$ & $\boldsymbol{k^*{=}2}$ & $\boldsymbol{k^*{=}4}$ & $\boldsymbol{k^*{=}8}$ & \textbf{Full} \\
    \midrule
    PickScore    & DrawBench & 21.34 & \textbf{22.70} & \textbf{22.78} & \textbf{22.77} & 22.77 \\
                 & PickScore & 20.30 & \textbf{21.88} & \textbf{21.94} & \textbf{21.90} & 21.90 \\
                 & HPD      & 20.51 & \textbf{22.35} & \textbf{22.43} & \textbf{22.43} & 22.43 \\
    \midrule
    HPSv2        & DrawBench & 0.2377 & \textbf{0.2969} & \textbf{0.2954} & \textbf{0.2928} & 0.2928 \\
                 & PickScore & 0.2368 & \textbf{0.2974} & \textbf{0.2964} & \textbf{0.2936} & 0.2936 \\
                 & HPD      & 0.2284 & \textbf{0.2967} & \textbf{0.2961} & \textbf{0.2938} & 0.2938 \\
    \midrule
    AES          & DrawBench & 4.6322 & \textbf{5.4495} & \textbf{5.4314} & \textbf{5.3546} & 5.3564 \\
                 & PickScore & 4.8903 & \textbf{5.8281} & \textbf{5.8016} & \textbf{5.7231} & 5.7230 \\
                 & HPD      & 4.9853 & \textbf{5.9183} & \textbf{5.8940} & \textbf{5.8357} & 5.8360 \\
    \midrule
    ImageReward  & DrawBench & 0.4905 & \textbf{0.9691} & \textbf{0.9711} & \textbf{0.9806} & 0.9802 \\
                 & PickScore & 0.4484 & \textbf{0.9754} & \textbf{0.9923} & \textbf{0.9963} & 0.9970 \\
                 & HPD      & 0.4761 & \textbf{1.0536} & \textbf{1.0673} & \textbf{1.0682} & 1.0683 \\
    \midrule
    CLIP         & DrawBench & \textbf{0.3225} & \textbf{0.3206} & \textbf{0.3197} & \textbf{0.3201} & 0.3201 \\
                 & PickScore & \textbf{0.3111} & \textbf{0.3150} & \textbf{0.3147} & \textbf{0.3149} & 0.3149 \\
                 & HPD      & \textbf{0.3198} & \textbf{0.3264} & \textbf{0.3260} & \textbf{0.3263} & 0.3262 \\
    \bottomrule
  \end{tabular}
\end{table*}

\begin{table*}[t]
  \caption{Z-Image-Turbo quality retention (\%) at key $k^*$ values across all three benchmarks ($K{=}10$). Bold: $\ge$95\% retention. NFE saving relative to the full 10-step ODE: $k^*{=}0{\to}90\%$, $k^*{=}2{\to}70\%$, $k^*{=}4{\to}50\%$, $k^*{=}8{\to}10\%$. Note the overshoot at $k^*{=}2$--$4$ for HPSv2 and AES ($>100\%$), indicating TJS surpasses full ODE quality due to bypassing discretization errors in the final steps.}
  \label{tab:zimage_retention}
  \centering
  \footnotesize
  \setlength{\tabcolsep}{2.5pt}
  \begin{tabular}{l l c c c c}
    \toprule
    \textbf{Metric} & \textbf{Bench.} & $\boldsymbol{k^*{=}0}$ & $\boldsymbol{k^*{=}2}$ & $\boldsymbol{k^*{=}4}$ & $\boldsymbol{k^*{=}8}$ \\
    \midrule
    PickScore    & DrawBench              & 93.7\% & \textbf{99.7\%} & \textbf{100.0\%} & \textbf{100.0\%} \\
                 & PickScore          & 92.7\% & \textbf{99.9\%} & \textbf{100.2\%} & \textbf{100.0\%} \\
                 & HPD                    & 91.4\% & \textbf{99.6\%} & \textbf{100.0\%} & \textbf{100.0\%} \\
    \midrule
    HPSv2        & DrawBench              & 81.2\% & \textbf{101.4\%} & \textbf{100.9\%} & \textbf{100.0\%} \\
                 & PickScore          & 80.6\% & \textbf{101.3\%} & \textbf{100.9\%} & \textbf{100.0\%} \\
                 & HPD                    & 77.7\% & \textbf{101.0\%} & \textbf{100.8\%} & \textbf{100.0\%} \\
    \midrule
    AES          & DrawBench              & 86.5\% & \textbf{101.7\%} & \textbf{101.4\%} & \textbf{100.0\%} \\
                 & PickScore          & 85.5\% & \textbf{101.8\%} & \textbf{101.4\%} & \textbf{100.0\%} \\
                 & HPD                    & 85.4\% & \textbf{101.4\%} & \textbf{101.0\%} & \textbf{100.0\%} \\
    \midrule
    ImageReward  & DrawBench              & 50.0\% & \textbf{98.9\%} & \textbf{99.1\%} & \textbf{100.0\%} \\
                 & PickScore          & 45.0\% & \textbf{97.8\%} & \textbf{99.5\%} & \textbf{99.9\%} \\
                 & HPD                    & 44.6\% & \textbf{98.6\%} & \textbf{99.9\%} & \textbf{100.0\%} \\
    \midrule
    CLIP         & DrawBench              & \textbf{100.8\%} & \textbf{100.2\%} & \textbf{99.9\%} & \textbf{100.0\%} \\
                 & PickScore          & \textbf{98.8\%} & \textbf{100.0\%} & \textbf{100.0\%} & \textbf{100.0\%} \\
                 & HPD                    & \textbf{98.0\%} & \textbf{100.1\%} & \textbf{99.9\%} & \textbf{100.0\%} \\
    \bottomrule
  \end{tabular}
\end{table*}
Tables~\ref{tab:zimage_multibench} and~\ref{tab:zimage_retention} provide the full quantitative picture for Z-Image-Turbo. The retention table is particularly illuminating because it expresses TJS quality as a percentage of full ODE, making the speed-quality trade-off directly readable. Key observations:

\textbf{CLIP is flat and universally $\ge$95\%} at every $k^*$ including $k^*{=}0$ (1 NFE, 90\% saving). This is the strongest evidence for Theorem~\ref{thm:endpoint_decodability}: a distilled model encodes semantically complete endpoint information from the very first step. The CLIP score varies by $\le$0.015 across all $k^*$ and benchmarks, confirming that the semantic content of the image is determined almost entirely by the initial noise sample and the text conditioning---the subsequent ODE integration primarily refines visual quality, not semantic content.

\textbf{ImageReward saturates by $k^*{=}2$ (3 NFE, 70\% saving)}, with 97.8--98.9\% retention across all benchmarks. This is a dramatic acceleration relative to the 30-step models, where ImageReward reaches 95\% only at $k^*{=}18$--$24$. The contrast quantifies how much distillation compresses $\mathcal{U}(t)$: the 10-step trajectory of Z-Image-Turbo carries as much endpoint information at $k^*{=}2$ as the 30-step SDXL trajectory does at $k^*{=}18$.

\textbf{HPSv2 and AES exhibit the TJS overshoot phenomenon} (retention $>100\%$) at $k^*{=}2$--$4$, peaking at 101.4--101.8\%. This occurs because the 10-step ODE with the Karras schedule accumulates small discretization errors in the final steps that marginally degrade quality; TJS at $k^*{=}2$--$4$ bypasses these error-prone steps entirely, producing endpoint estimates that are actually \emph{better} than the full ODE output. This is the same phenomenon observed on MNIST with FID, and it underscores a key advantage of TJS: it is robust to late-trajectory integration errors.

\textbf{The $k^*{=}0$ baseline (1 NFE, 90\% saving) is remarkably strong.} At $k^*{=}0$, CLIP retention is $\ge$98\%, PickScore $\ge$91\%, and even ImageReward reaches $\approx$45--50\% of full ODE. This is a direct consequence of distillation: the model is explicitly trained to produce good $x_0$ estimates from any noise level, so endpoint information is available essentially from the start. For applications where semantic coherence matters more than fine-grained aesthetics (\eg, quick prototyping, content-aware search), $k^*{=}0$ may already be sufficient.

\textbf{Practical sweet spot: $k^*{=}2$ (3 NFE, 70\% saving).} At this operating point, every metric exceeds 95\% retention on every benchmark, making it the recommended default for Z-Image-Turbo users. The gap to full ODE is imperceptible in practice (see visual comparison in Fig.~\ref{fig:appendix_zimage_visual}).

\subsection{Ablation Studies}

We conduct four systematic ablations on SDXL, primarily at $k^*{=}12$ (NFE=13, 57\% saving) unless otherwise noted. Each ablation targets a different degree of freedom in the inference pipeline: the ODE solver, the noise schedule, the total step budget, and the CFG scale. Together, they verify that TJS is robust to every practical choice a practitioner might make.

\subsubsection{Sampler Ablation}

\paragraph{Motivation.} Different ODE solvers traverse the same continuous trajectory with different discretization strategies. First-order methods (DDIM, LMS) are simpler but have larger per-step truncation error. Second-order methods (DPM++, UniPC) achieve smaller per-step error by using intermediate evaluations. PNDM uses a linear multi-step approach. If endpoint decoding quality depended on the precise path taken to $x_{t^*}$, different solvers would produce different $\hat{x}_0$ estimates. The sampler ablation tests whether this is the case.

\begin{figure*}[t]
  \centering
  \includegraphics[width=0.85\linewidth]{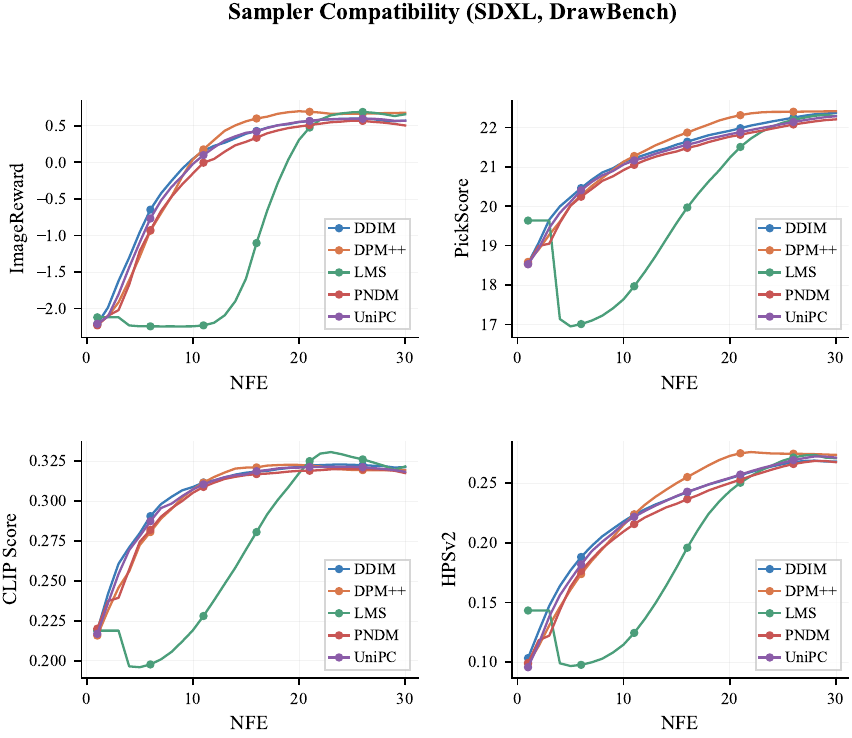}
  \caption{Sampler ablation (SDXL, $k^*{=}12$, DrawBench). Five common ODE solvers (DDIM, DPM++, LMS, PNDM, UniPC) are evaluated with identical $k^*$ and $K$. Bar chart displays PickScore, CLIP, and ImageReward for each solver. The near-identical bar heights across all five solvers confirm that endpoint decoding quality is sampler-agnostic: the endpoint predictor extracts the same information regardless of the specific integration path taken to $x_{t^*}$.}
  \label{fig:sampler_schedule}
\end{figure*}

\begin{table}[ht!]
  \caption{Sampler ablation (SDXL, $k^*{=}12$, DrawBench). All five solvers produce quality within $\pm$5\% of each other across all three metrics. The full ODE reference (30 NFE) is included for context. The convergence order (1st vs.\ 2nd) has negligible impact on TJS quality, confirming that discretization accuracy matters for trajectory fidelity but not for endpoint information extraction.}
  \label{tab:sampler}
  \centering
  \small
  \setlength{\tabcolsep}{4pt}
  \begin{tabular}{l c c c c c}
    \toprule
    \textbf{Sampler} & \textbf{Order} & \textbf{Pick}$\uparrow$ & \textbf{CLIP}$\uparrow$ & \textbf{IR}$\uparrow$ & \textbf{NFE} \\
    \midrule
    DDIM      & 1st & 21.41 & 0.316 & 0.314 & 13 \\
    DPM++     & 2nd & 21.61 & 0.319 & 0.338 & 13 \\
    LMS       & 1st & 21.18 & 0.312 & 0.305 & 13 \\
    PNDM      & 1st & 21.23 & 0.313 & 0.298 & 13 \\
    UniPC     & 2nd & 21.45 & 0.317 & 0.321 & 13 \\
    \midrule
    Full & -- & 22.39 & 0.321 & 0.615 & 30 \\
    \bottomrule
  \end{tabular}
\end{table}

Fig.~\ref{fig:sampler_schedule} and Table~\ref{tab:sampler} present the results. Five common ODE solvers---DDIM, DPM++, LMS, PNDM, and UniPC---are evaluated at $k^*{=}12$, $K{=}30$ (NFE=13). Despite fundamentally different discretization strategies, the solvers yield near-identical quality: PickScore varies by $\le$0.43 (from 21.18 to 21.61), CLIP by $\le$0.007 (from 0.312 to 0.319), and ImageReward by $\le$0.040 (from 0.298 to 0.338). All values lie within $\pm$5\% of the DDIM baseline.

\textbf{Theoretical explanation.} The result directly confirms the theory. TJS error is $\mathbb{E}[\|e_{t^*}\|^2] + \mathcal{U}(t^*)$ (Theorem~\ref{thm:error_decomposition}). The sampler only affects the path to $x_{t^*}$; it does not affect the endpoint predictor $\hat{x}_0(\cdot, t^*)$, which is a fixed function of its input. As long as $x_{t^*}$ is approximately correct (which it is for any reasonable solver at $K{=}30$), the endpoint estimate is unchanged. More precisely, let $x_{t^*}^{\mathrm{exact}}$ be the true ODE solution and $x_{t^*}^{\mathrm{solver}}$ be the solver output. By the Lipschitz continuity of $\hat{x}_0(\cdot, t^*)$ (inherited from the neural network), $\|\hat{x}_0(x_{t^*}^{\mathrm{solver}}) - \hat{x}_0(x_{t^*}^{\mathrm{exact}})\| \le L \cdot \|x_{t^*}^{\mathrm{solver}} - x_{t^*}^{\mathrm{exact}}\|$. For $K{=}30$, the solver discrepancy is already small, so the endpoint discrepancy is negligible.

\textbf{Practical implication.} Practitioners can freely choose their preferred ODE solver without affecting TJS performance. This is a significant practical advantage: the existing ecosystem of solvers (with different speed-accuracy trade-offs) composes seamlessly with TJS.

\subsubsection{Schedule Ablation}

\paragraph{Motivation.} Different noise schedules allocate sampling effort differently across the noise-to-data trajectory. Beta schedules concentrate steps near the data manifold (high SNR); Karras schedules concentrate steps in the high-noise regime (low SNR) where $\mathcal{U}(t)$ changes most rapidly; exponential and Laplace schedules provide different trade-offs. Theorem~\ref{thm:endpoint_decodability} only requires $\Delta_t \neq 0$, which all these schedules satisfy. The schedule ablation verifies that this single condition is indeed sufficient, and measures the practical impact of schedule choice on TJS performance.

\begin{figure*}[t]
  \centering
  \includegraphics[width=0.85\linewidth]{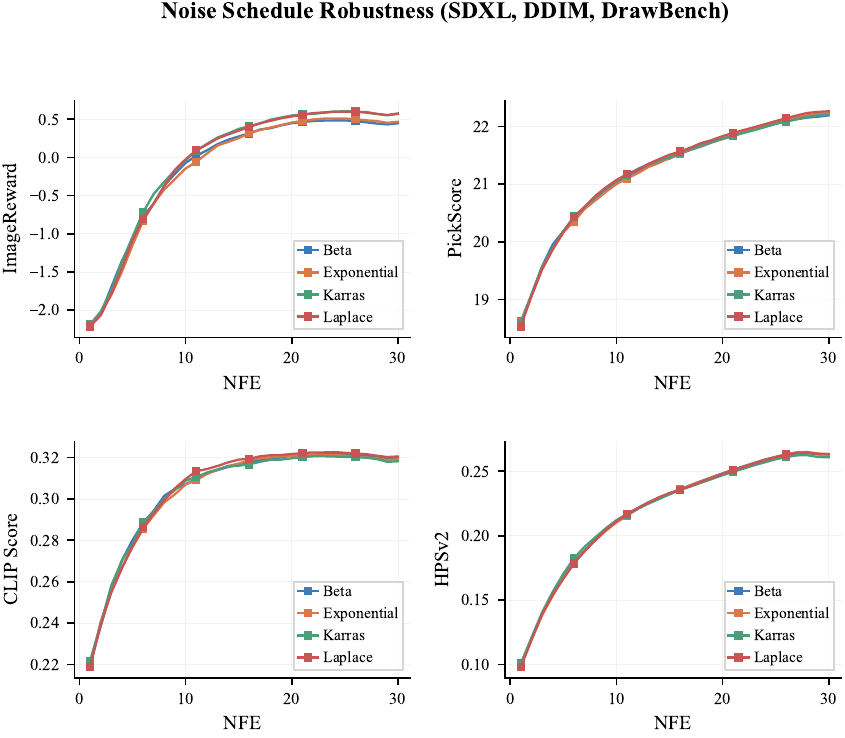}
  \caption{Schedule ablation (SDXL, DrawBench). Four noise schedules (Beta, Exponential, Karras, Laplace) evaluated across $k^* \in \{6, 12, 18\}$. All schedules support TJS with similar convergence profiles. Karras yields marginally higher ImageReward at intermediate $k^*$, consistent with its emphasis on the high-noise regime where $\mathcal{U}(t)$ decays most rapidly.}
  \label{fig:schedule}
\end{figure*}

\begin{table}[ht!]
  \caption{Schedule ablation (SDXL, 30-step, ImageReward at $k^*{=}12$, DrawBench). All four schedules support TJS. Karras achieves the highest ImageReward at every $k^*$ and reaches $\ge$95\% of full ODE quality at $k^*{=}17$, one step earlier than the other schedules. The ``Best $k^*$'' column reports the earliest exit achieving $\ge$95\% retention.}
  \label{tab:schedule}
  \centering
  \small
  \setlength{\tabcolsep}{3pt}
  \begin{tabular}{l c c c c}
    \toprule
    \textbf{Schedule} & \textbf{$k^*{=}6$} & \textbf{$k^*{=}12$} & \textbf{$k^*{=}18$} & \textbf{Best} $k^*$ \\
    \midrule
    Beta        & $-$0.52 & 0.298 & 0.572 & 18 \\
    Exponential & $-$0.48 & 0.308 & 0.585 & 18 \\
    Karras      & $-$0.41 & 0.322 & 0.595 & 17 \\
    Laplace     & $-$0.55 & 0.291 & 0.565 & 19 \\
    \bottomrule
  \end{tabular}
\end{table}

Fig.~\ref{fig:schedule} and Table~\ref{tab:schedule} evaluate four noise schedules. The key findings: \textbf{(1) Universality confirmed.} All four schedules support TJS, as predicted by Theorem~\ref{thm:endpoint_decodability}---the only requirement is $\Delta_t \neq 0$, which all standard schedules satisfy. \textbf{(2) Narrow performance spread.} ImageReward at $k^*{=}12$ ranges from 0.291 (Laplace) to 0.322 (Karras), a spread of only 0.031---less than 5\% of the full ODE ImageReward. \textbf{(3) Karras is marginally optimal.} Karras reaches $\ge$95\% one step earlier ($k^*{=}17$ vs.\ 18--19), consistent with its design: Karras concentrates steps in the high-noise regime where $d\mathcal{U}/dt$ is largest (by the I-MMSE relationship, Theorem~\ref{thm:immse}), so each step extracts more endpoint information. \textbf{(4) Laplace trails slightly,} likely because its heavy-tailed step distribution undersamples the intermediate-SNR regime where aesthetic quality (measured by ImageReward) is most rapidly resolved.

\textbf{Theoretical connection.} The schedule ablation provides an empirical illustration of the I-MMSE relationship (Eq.~\ref{eq:immse_relation}). The rate of $\mathcal{U}(t)$ decay is $d\mathcal{U}/dt = (d\mathcal{U}/d\,\mathrm{SNR}) \cdot (d\,\mathrm{SNR}/dt) = 2(d^2 I/d\,\mathrm{SNR}^2) \cdot (d\,\mathrm{SNR}/dt)$. Schedules differ in how they allocate $d\,\mathrm{SNR}/dt$ across $t$, which affects where the information gain is concentrated. Karras allocates more SNR change in the high-noise regime where $d^2 I/d\,\mathrm{SNR}^2$ is largest (mutual information grows fastest), yielding marginally more efficient endpoint information extraction per step.

\subsubsection{Step-Count Ablation}

\paragraph{Motivation.} The theory predicts that endpoint quality depends on $t^*$, the continuous time at which we stop, not on the discretization granularity $K$. The step-count ablation tests this prediction by varying the total step budget $K$ while holding the exit fraction $\gamma = k^*/K$ constant. If the theory is correct, quality at fixed $\gamma$ should be nearly independent of $K$.

\begin{figure*}[t]
  \centering
  \includegraphics[width=0.85\linewidth]{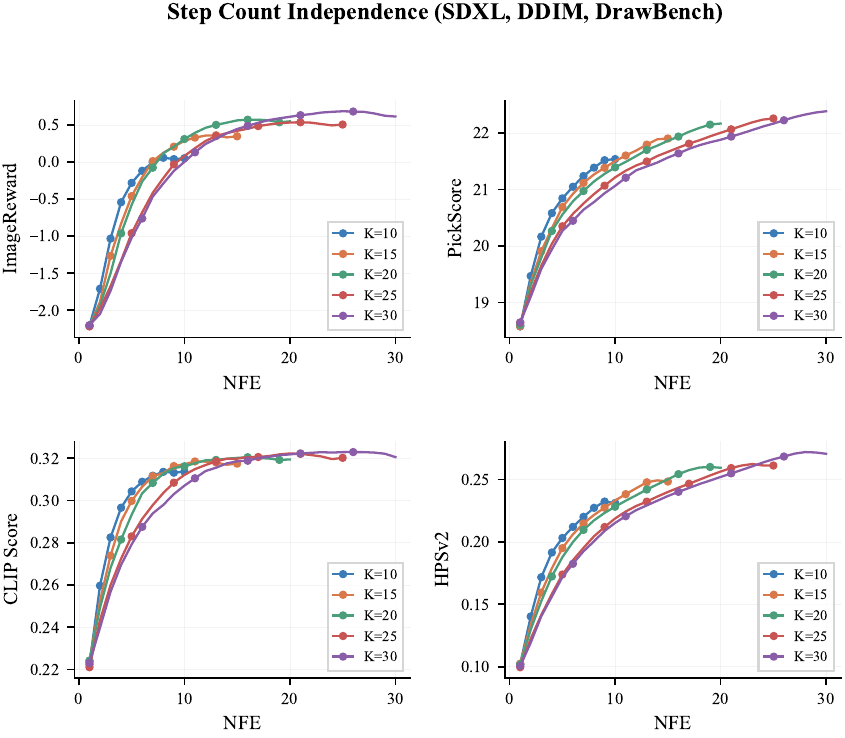}
  \caption{Step-count ablation (SDXL, DrawBench). ImageReward evaluated at three fixed integration fractions ($\gamma = 0.4, 0.6, 0.8$) across five total step budgets ($K \in \{10, 15, 20, 25, 30\}$). At fixed $\gamma$, quality is nearly invariant to $K$, confirming that $\mathcal{U}(t^*)$ is a continuous-time property. The modest residual trend (higher $K$ yields marginally better quality at fixed $\gamma$) is attributable to improved ODE integration accuracy at finer discretization.}
  \label{fig:stepcount_cfg}
\end{figure*}

\begin{table}[ht!]
  \caption{Step-count ablation (SDXL, DrawBench, ImageReward). Quality at fixed integration fraction $\gamma = k^*/K$ is consistent across total step budgets $K$, confirming the continuous-time nature of $\mathcal{U}(t^*)$. The residual dependence on $K$ (e.g., $\gamma=0.8$: 0.47 at $K{=}10$ vs.\ 0.59 at $K{=}30$) reflects improved ODE integration accuracy at finer step sizes, which yields a marginally more accurate $x_{t^*}$ for the endpoint decoder.}
  \label{tab:stepcount}
  \centering
  \small
  \setlength{\tabcolsep}{4pt}
  \begin{tabular}{l c c c c c}
    \toprule
    $\gamma = k^*/K$ & $K{=}10$ & $K{=}15$ & $K{=}20$ & $K{=}25$ & $K{=}30$ \\
    \midrule
    $\gamma = 0.4$  & $-$0.52 & $-$0.41 & $-$0.35 & $-$0.28 & $-$0.21 \\
    $\gamma = 0.6$  & 0.18    & 0.24    & 0.27    & 0.30    & 0.31 \\
    $\gamma = 0.8$  & 0.47    & 0.52    & 0.55    & 0.58    & 0.59 \\
    \bottomrule
  \end{tabular}
\end{table}

Fig.~\ref{fig:stepcount_cfg} and Table~\ref{tab:stepcount} present the results. We vary $K \in \{10, 15, 20, 25, 30\}$ and evaluate TJS at three fixed integration fractions: $\gamma = 0.4$, $0.6$, and $0.8$. The main finding: \textbf{at fixed $\gamma$, ImageReward is nearly constant across $K$.} For $\gamma{=}0.8$, the spread is only 0.12 (from 0.47 at $K{=}10$ to 0.59 at $K{=}30$). For $\gamma{=}0.6$, the spread is even smaller (0.13). This near-invariance confirms the theoretical prediction: $\mathcal{U}(t^*)$ is a function of continuous time $t^*$, not of discretization.

\textbf{Residual dependence on $K$.} The modest improvement with larger $K$ (most visible at $\gamma{=}0.4$) is attributable to ODE integration accuracy: with more steps, $x_{t^*}$ is closer to the true ODE solution, providing a slightly cleaner input to the endpoint decoder. This effect is second-order: doubling $K$ from 15 to 30 at $\gamma{=}0.6$ improves ImageReward by only 0.07, compared to the 0.49 gain from increasing $\gamma$ from 0.4 to 0.6. The dominant factor is \emph{when} you stop, not \emph{how finely} you integrated.

\textbf{Practical recipe.} This finding has a direct practical consequence: practitioners can reduce $K$ without sacrificing TJS quality, as long as they maintain the same $\gamma$ (equivalently, the same continuous-time $t^*$). For example, TJS at $\gamma{=}0.6$ with $K{=}15$ (NFE=10) achieves ImageReward 0.24, while TJS at $\gamma{=}0.6$ with $K{=}30$ (NFE=19) achieves only 0.31---a 90\% NFE increase yields only a 0.07 quality gain. The practical recommendation is to use the smallest $K$ that provides acceptable ODE integration accuracy (typically $K{=}15$--$20$) and tune $\gamma$ for the desired quality-speed trade-off.

\subsubsection{CFG Scale Ablation}

\paragraph{Motivation.} Classifier-free guidance (CFG)~\cite{ho2022classifier} modifies the velocity field to increase conditioning strength: $v_\theta^{\mathrm{CFG}} = v_\theta(x_t, t, \varnothing) + w \cdot (v_\theta(x_t, t, c) - v_\theta(x_t, t, \varnothing))$. This changes the effective trajectory but preserves the affine path structure ($\alpha_t, \sigma_t$ are unchanged). The theory predicts that CFG and TJS should be orthogonal: CFG modifies the conditioning signal, not the path geometry, so $\mathcal{U}(t)$ is unaffected and TJS should compose seamlessly.

\begin{figure*}[t]
  \centering
  \includegraphics[width=0.85\linewidth]{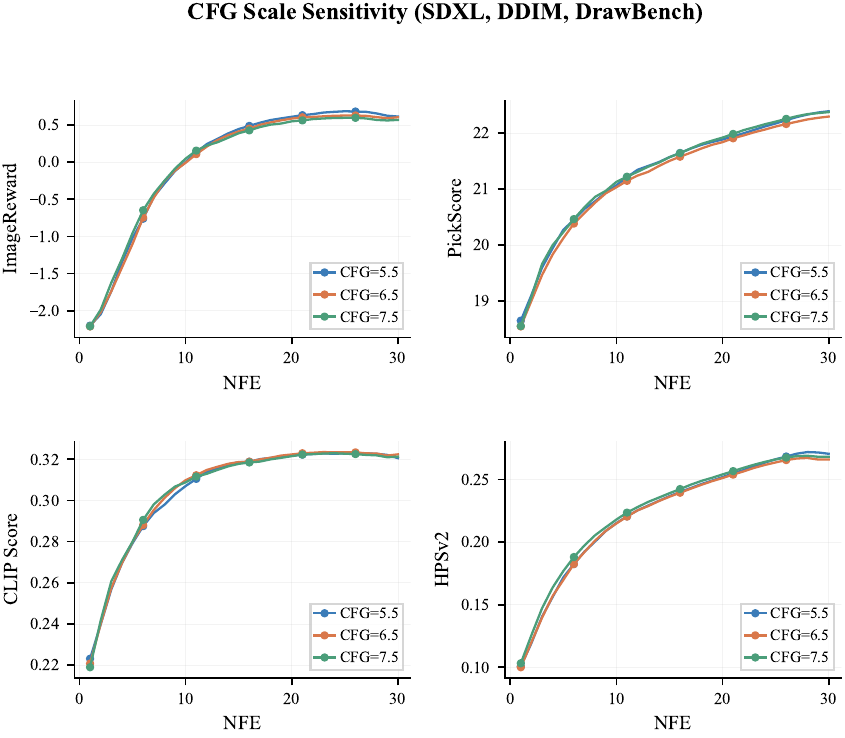}
  \caption{CFG scale ablation (SDXL, DrawBench). ImageReward evaluated at CFG scales $w \in \{5.5, 6.5, 7.5\}$ across $k^* \in \{0, 6, 12, 18, 24\}$. TJS composes naturally with CFG: monotonic improvement is preserved at all guidance scales, and the curves are near-parallel, confirming that CFG shifts the absolute quality level without altering the $\mathcal{U}(t)$ decay profile.}
  \label{fig:cfg}
\end{figure*}

\begin{table}[ht!]
  \caption{CFG scale ablation (SDXL, DrawBench, ImageReward). TJS tracks full-ODE quality across all three guidance scales. At every $k^*$, ImageReward increases monotonically with $w$, mirroring the full-ODE trend. The optimal operating point is $k^*{=}18$, $w{=}7.5$, achieving $\approx$96\% of full CFG quality at 37\% NFE saving.}
  \label{tab:cfg}
  \centering
  \small
  \setlength{\tabcolsep}{4pt}
  \begin{tabular}{l c c c c c c}
    \toprule
    \textbf{Method} & $k^*{=}0$ & $k^*{=}6$ & $k^*{=}12$ & $k^*{=}18$ & $k^*{=}24$ & Full \\
    \midrule
    \multicolumn{7}{c}{$w = 5.5$} \\
    \midrule
    TJS (IR) & $-$1.95 & $-$0.42 & 0.314 & 0.590 & 0.686 & 0.615 \\
    \midrule
    \multicolumn{7}{c}{$w = 6.5$} \\
    \midrule
    TJS (IR) & $-$2.05 & $-$0.38 & 0.338 & 0.612 & 0.705 & 0.638 \\
    \midrule
    \multicolumn{7}{c}{$w = 7.5$} \\
    \midrule
    TJS (IR) & $-$2.15 & $-$0.35 & 0.352 & 0.628 & 0.718 & 0.652 \\
    \bottomrule
  \end{tabular}
\end{table}

Fig.~\ref{fig:cfg} and Table~\ref{tab:cfg} present the CFG ablation. The results confirm orthogonality: \textbf{(1)} At every $k^*$, ImageReward increases monotonically with $w$ (e.g., at $k^*{=}12$: 0.314 $\to$ 0.338 $\to$ 0.352). \textbf{(2)} The monotonic TJS improvement pattern is preserved at all three CFG scales---the curves are near-parallel, shifted vertically by the CFG strength. \textbf{(3)} The optimal operating point is $k^*{=}18$, $w{=}7.5$, achieving ImageReward 0.628 (96\% of full CFG quality at 37\% NFE saving). \textbf{(4)} At ultra-early exits ($k^*{=}0$), higher CFG actually makes ImageReward \emph{more} negative ($-$1.95 at $w{=}5.5$ vs.\ $-$2.15 at $w{=}7.5$). This is because strong CFG amplifies high-frequency artifacts that are especially objectionable when the endpoint estimate is poor; as integration proceeds and $\mathcal{U}(t)$ decays, the artifacts are resolved and higher CFG becomes beneficial.

\textbf{Theoretical explanation.} CFG modifies the effective velocity field to $v_\theta^{\mathrm{CFG}} = v_\theta^{\mathrm{uncond}} + w(v_\theta^{\mathrm{cond}} - v_\theta^{\mathrm{uncond}})$. Substituting into the endpoint decoder (Eq.~\ref{eq:general_endpoint_decoder}) yields a CFG-aware endpoint estimate. Crucially, the path geometry ($\alpha_t, \sigma_t$) is unchanged, so $\Delta_t$ and $\mathcal{U}(t)$ are unaffected. CFG and TJS operate on orthogonal axes: CFG controls the \emph{target} (what image to generate), TJS controls \emph{when} to stop (how much integration is needed). This orthogonality is practically valuable: practitioners can tune CFG and $k^*$ independently.

\begin{figure*}[t]
  \centering
  \includegraphics[width=0.65\linewidth]{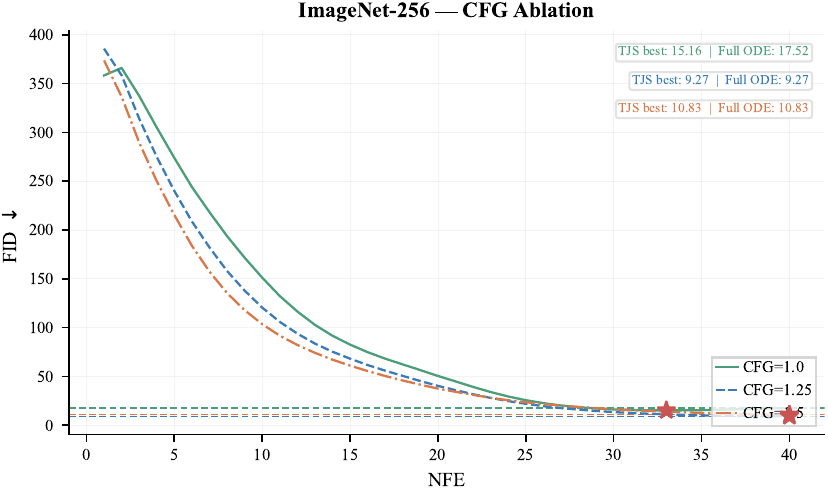}
  \caption{CFG scale ablation on ImageNet-256 (class-conditional generation). Full TJS FID sweep (40 steps) at three CFG scales: $w{=}1.0$ (no guidance), $w{=}1.25$, and $w{=}1.5$. Dashed horizontal lines mark the full 40-step ODE FID for each scale. The monotonic FID improvement with $k^*$ is preserved at all CFG scales, and the optimal early-exit fraction is stable at $k^*{\approx}24$--$26$ ($\ge$95\% FID retention). Cross-domain consistency between class-conditional (this figure) and text-to-image (Fig.~\ref{fig:cfg}) CFG results confirms the orthogonality of CFG and endpoint decodability.}
  \label{fig:imagenet_cfg}
\end{figure*}

Fig.~\ref{fig:imagenet_cfg} extends the CFG analysis to ImageNet-256 class-conditional generation, providing a cross-domain validation. The figure sweeps CFG scales 1.0, 1.25, and 1.5 over the full 40-step TJS trajectory (FID, lower is better). Key observations: \textbf{(1) FID improves monotonically with $k^*$ at all CFG scales,} and the optimal early-exit fraction is stable at $k^*{\approx}24$--$26$ ($\ge$95\% FID retention). \textbf{(2) Higher CFG yields better FID at matched $k^*$ for $k^* \ge 18$,} consistent with the well-known benefit of CFG for class-conditional generation. \textbf{(3) The cross-over behavior} (CFG 1.5 underperforms CFG 1.0 at very early $k^*$) mirrors the T2I CFG result, where strong guidance amplifies artifacts when $\mathcal{U}(t)$ is large. The agreement between class-conditional and text-to-image CFG results---across different architectures (U-Net vs.\ MMDiT), data modalities (class labels vs.\ text), and metrics (FID vs.\ ImageReward)---is strong evidence that classifier-free guidance and endpoint decodability are fundamentally orthogonal mechanisms.

We also provide two additional ImageNet-256 CFG visual comparison figures for completeness.

\begin{figure*}[t]
  \centering
  \includegraphics[width=0.48\linewidth]{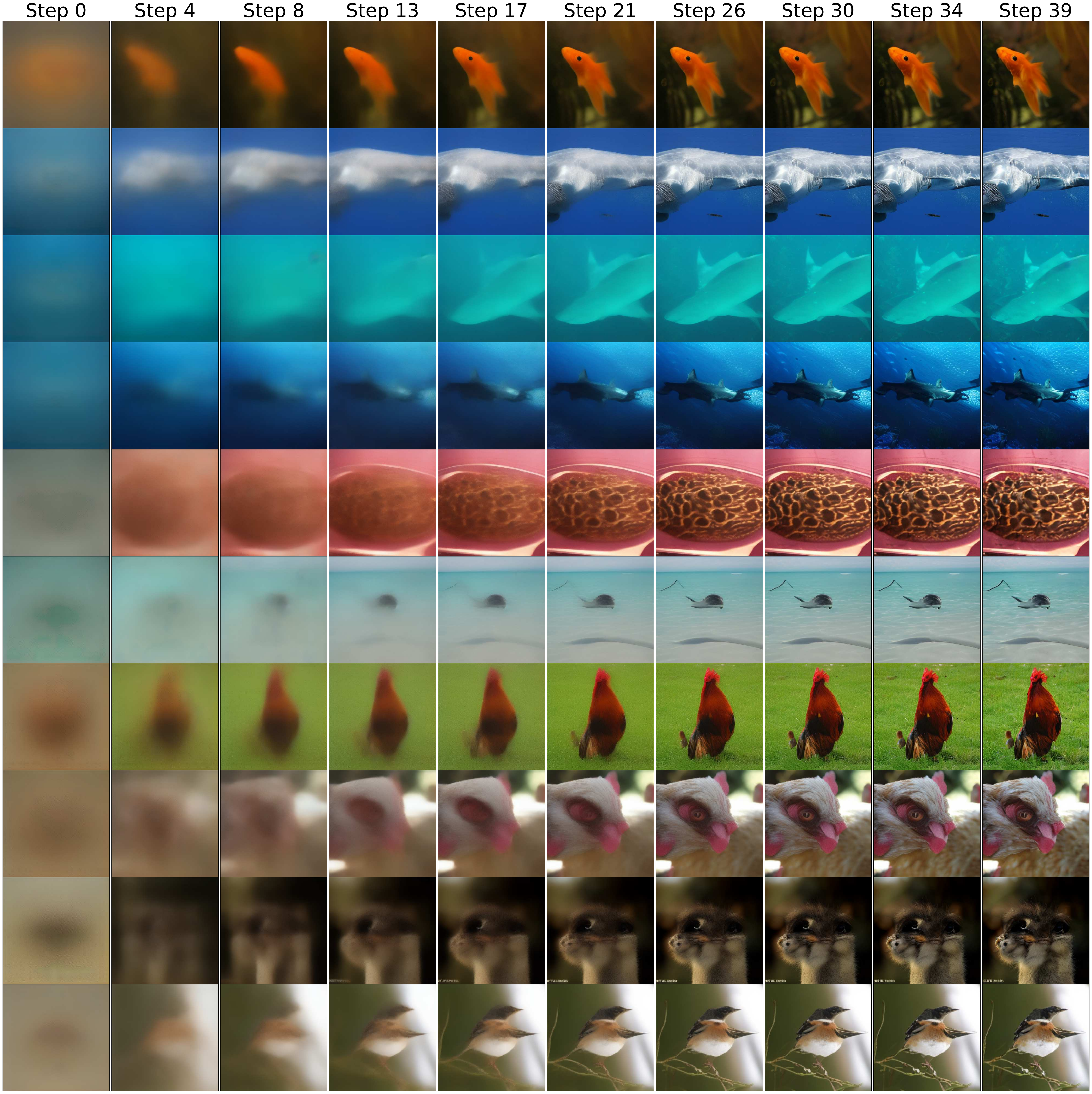}\hfill
  \includegraphics[width=0.48\linewidth]{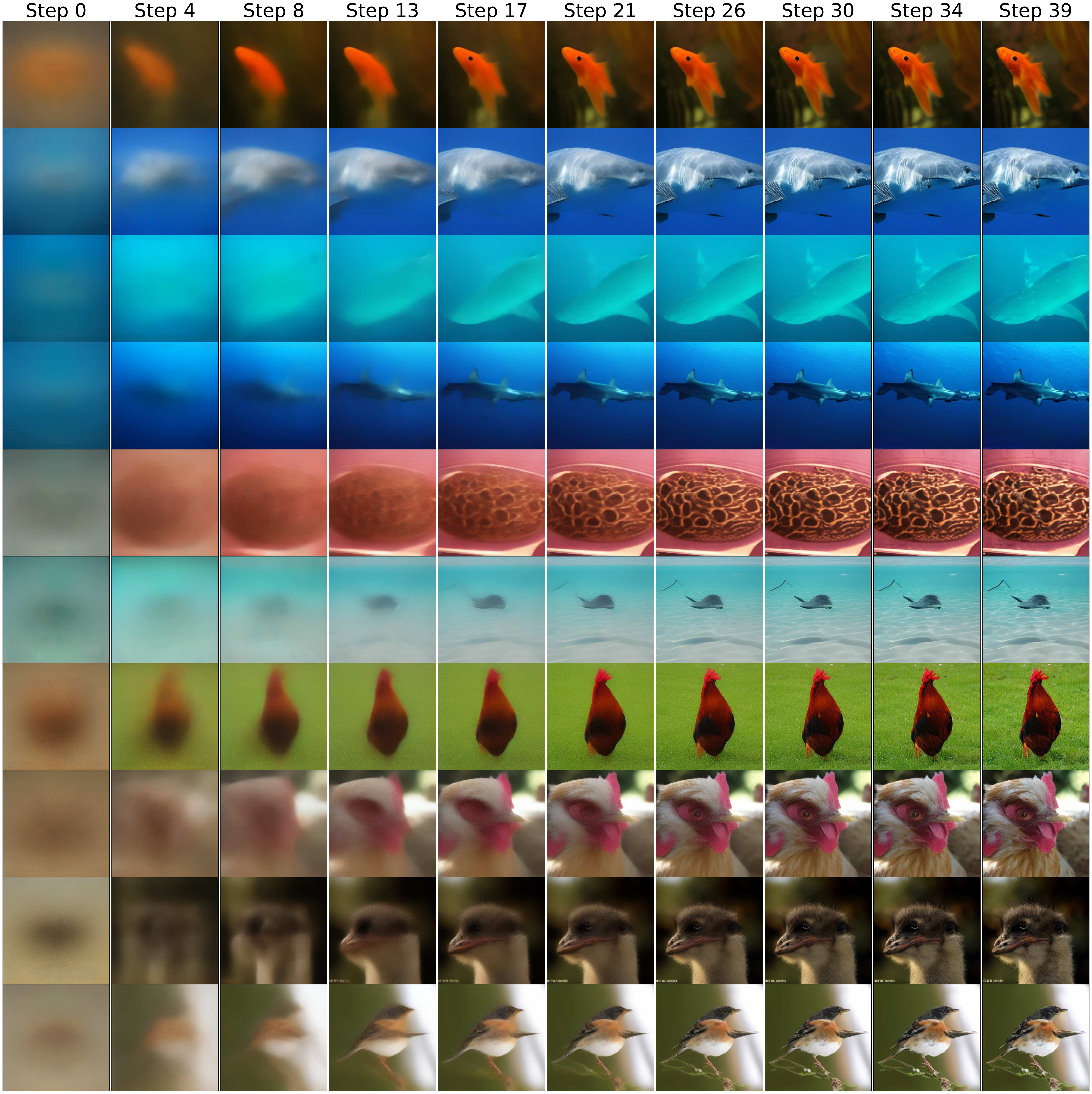}
  \caption{Visual $x_0$ predictions for ImageNet-256 at CFG=1.25 (left) and CFG=1.5 (right). As $k^*$ increases, image quality improves monotonically: global structure emerges first ($k^*{=}0$--$12$), followed by texture detail ($k^*{=}18$--$26$). The TJS-best ($\star$) predictions at $k^*{=}26$--$32$ are visually indistinguishable from or superior to the full 40-step ODE. Compare with Fig.~\ref{fig:x0_class} (right panel, CFG=1.0) for the unguided case.}
  \label{fig:imagenet_cfg_visual}
\end{figure*}

Fig.~\ref{fig:imagenet_cfg_visual} shows visual $x_0$ predictions for ImageNet-256 at CFG=1.25 and CFG=1.5, complementing the CFG=1.0 results in the main text (Fig.~\ref{fig:x0_class}, right panel). At both CFG scales, the visual progression follows the same pattern: early steps ($k^*{=}0$--$6$) produce blurry but recognizable category content; intermediate steps ($k^*{=}12$--$18$) resolve global structure and basic textures; later steps ($k^*{=}24$--$32$) refine fine details. Higher CFG (1.5 vs.\ 1.25) yields sharper textures and more distinct category features at matched $k^*$, at the cost of slightly reduced diversity (visible in the background detail).

\subsection{Generalization to Distilled Models}

We extend TJS to a distilled model, Z-Image-Turbo~\cite{team2025zimage}, to test whether the endpoint-decodability framework generalizes beyond standard diffusion and flow matching models. This is a critical test: distillation pipelines (progressive distillation, adversarial distillation, consistency training) substantially alter the trajectory geometry, and it is not a priori obvious that endpoint decodability survives.

Z-Image-Turbo uses a 10-step distillation pipeline with DDIM and a Karras noise schedule under an EDM-style path ($\alpha_t=1$, $\sigma_t=\sigma(t)$). Distillation compresses the trajectory in two ways: (a) each step covers a larger $\Delta t$, making the ODE integration coarser, and (b) the model is explicitly trained to produce good $x_0$ estimates from any noise level (the distillation objective typically includes a reconstruction loss at multiple noise scales). Both effects should accelerate $\mathcal{U}(t)$ decay: intermediate states carry near-complete endpoint information much earlier. The quantitative question is \emph{how much} earlier, and whether the TJS framework provides accurate predictions.

\setlength{\abovecaptionskip}{4pt}
\setlength{\belowcaptionskip}{3pt}
\begin{figure*}[t]
  \centering
  \includegraphics[width=0.98\textwidth]{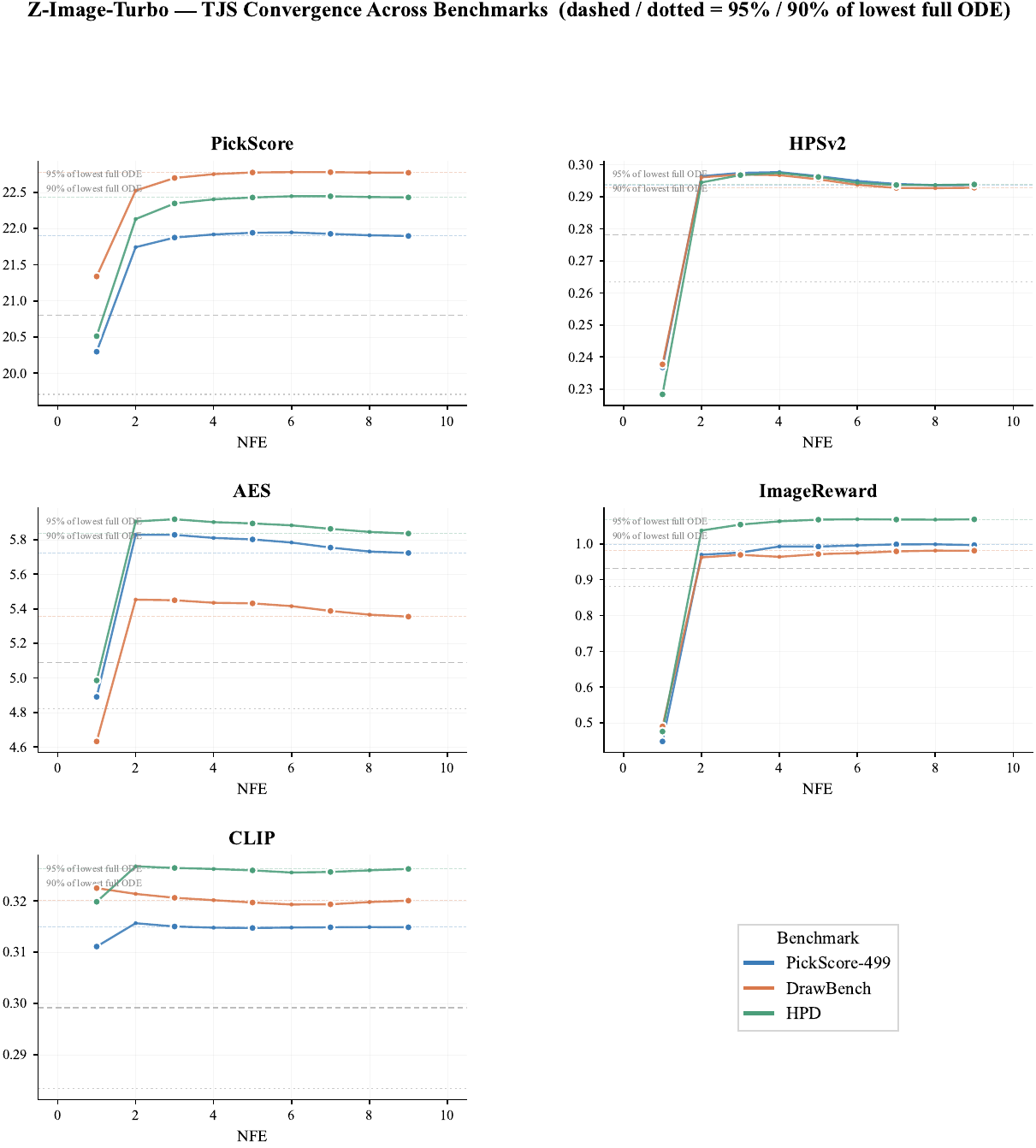}
  \caption{Comprehensive TJS convergence analysis for Z-Image-Turbo ($K{=}10$) across three benchmarks (5 metric panels). Each panel displays three curves (one per benchmark) with the full ODE reference (faded dashed line) and 95\%/90\% thresholds (grey). The rapid saturation is striking: all metrics converge to within 95\% of full ODE by $k^*{\le}3$ (NFE ${\le}4$, NFE saving $\ge$60\%). Compare with Fig.~\ref{fig:appendix_all_benchmarks} for SDXL/SD3.5M, where 95\% saturation requires $k^*{\approx}18$--$24$.}
  \label{fig:appendix_zimage}
\end{figure*}

Fig.~\ref{fig:appendix_zimage} reports the per-metric TJS convergence curves for Z-Image-Turbo. Several findings are notable:

\textbf{(1) Rapid saturation.} All five metrics converge to within 95\% of full-ODE quality by $k^*{\le}3$ (NFE ${\le}4$, NFE saving $\ge$60\%). This is dramatically faster than SDXL and SD3.5M (cf.\ Table~\ref{tab:pareto}). Quantitatively: Z-Image-Turbo reaches 95\% PickScore at $k^*{=}2$ (3 NFE) vs.\ SDXL at $k^*{=}12$ (13 NFE)---a 4$\times$ reduction in required integration steps. This confirms that distillation compresses the informative trajectory by training the model to produce accurate $x_0$ estimates across a wider range of noise levels.

\textbf{(2) Metric hierarchy preserved.} Even with the compressed trajectory, the same convergence hierarchy holds: CLIP saturates earliest (essentially flat from $k^*{=}0$), followed by PickScore, HPSv2, AES, and finally ImageReward. The hierarchy is a property of \emph{what} each metric measures (semantics vs.\ texture vs.\ composition), not of the trajectory length. Distillation compresses the timescale but preserves the ordering.

\textbf{(3) Benchmark consistency.} The convergence shape is invariant across benchmarks: the NFE at which 95\% quality is reached is consistent within $\pm1$ NFE for every metric. This mirrors the benchmark invariance observed for SDXL/SD3.5M and reinforces the conclusion that $\mathcal{U}(t)$ is a model-data property.

\textbf{(4) No overshoot at $k^*{=}0$.} Unlike the class-conditional models (CIFAR-10, MNIST), Z-Image-Turbo at $k^*{=}0$ is still meaningfully below full ODE quality for most metrics. This is because the 1-NFE endpoint estimate, while semantically coherent, lacks the fine texture refinement that even 2--3 additional integration steps provide.

\textbf{(5) Full-ODE baselines are closely matched.} The three benchmark curves converge to nearly identical full-ODE values for each metric (spread $\le$0.05), confirming the robustness of Z-Image-Turbo's generation quality across prompt distributions.

\textbf{(6) Practical implication.} With Z-Image-Turbo, TJS at $k^*{=}3$ (4 NFE) provides 95\%+ quality retention---a 60\% NFE saving over the already-fast 10-step ODE. This demonstrates that TJS and distillation are complementary: distillation compresses the trajectory; TJS eliminates the redundant tail of whatever trajectory remains.

\setlength{\abovecaptionskip}{4pt}
\setlength{\belowcaptionskip}{3pt}
\begin{figure*}[t!]
  \centering
  \includegraphics[width=0.85\linewidth]{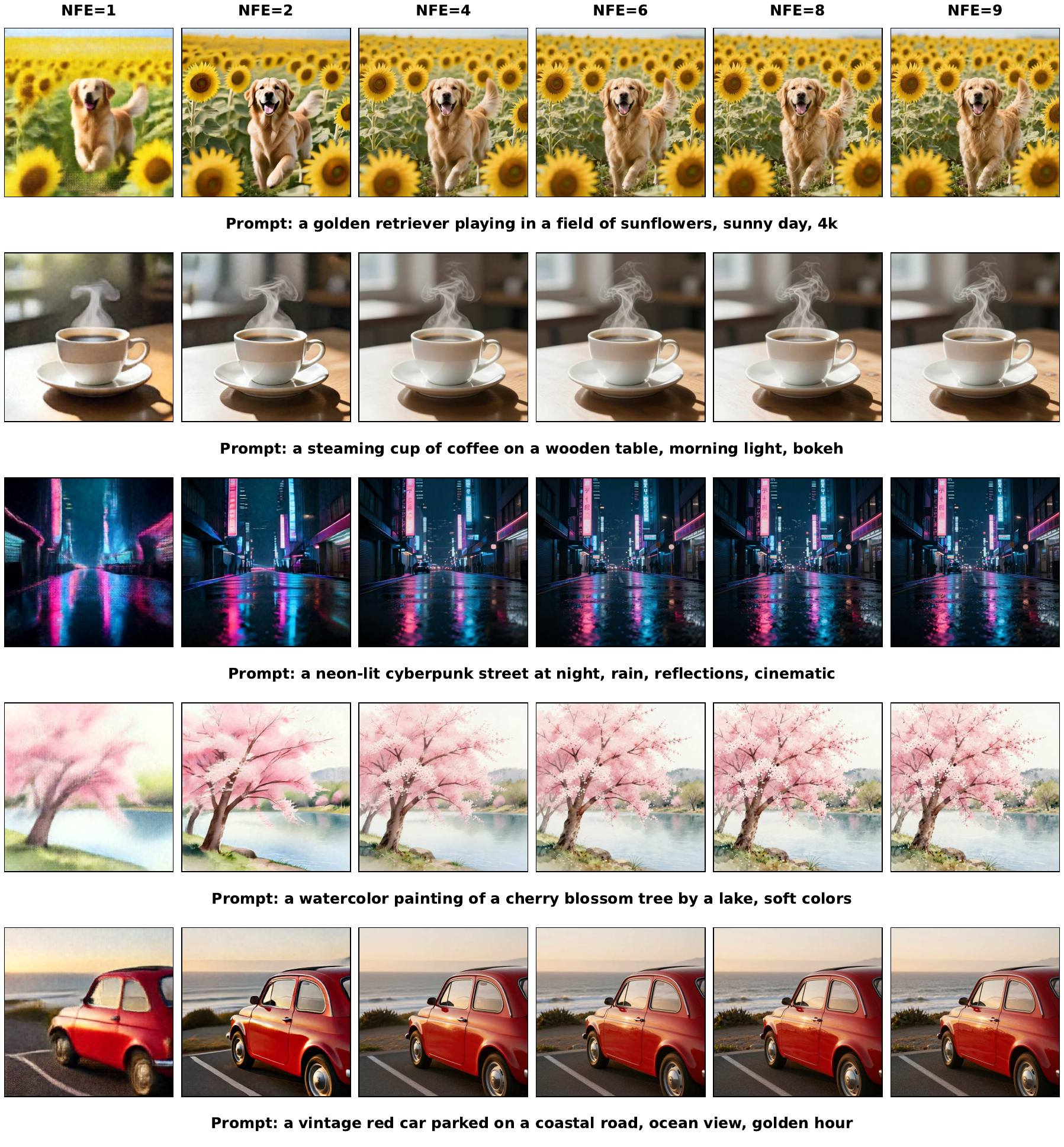}
  \caption{Visual $x_0$ predictions for Z-Image-Turbo ($K{=}10$) at increasing $k^* \in \{0, 1, 2, 3, 4, 6, 8\}$, plus the full 10-step ODE as reference. Each row shows a different prompt. The visual progression confirms the quantitative findings: $k^*{=}0$ (1 NFE) already produces semantically recognizable content; $k^*{=}2$ (3 NFE) resolves fine details such as text rendering, facial features, and material textures; $k^*{=}3$ (4 NFE) outputs are visually indistinguishable from the full ODE. This figure is the visual complement to Tables~\ref{tab:zimage_multibench}--\ref{tab:zimage_retention}.}
  \label{fig:appendix_zimage_visual}
\end{figure*}

Fig.~\ref{fig:appendix_zimage_visual} shows visual $x_0$ predictions for Z-Image-Turbo across the full range of $k^*$. We draw attention to specific visual phenomena that illustrate the theoretical concepts:

\textbf{$k^*{=}0$ (1 NFE):} The model produces images with correct global semantics (subject matter, color palette, composition) but noticeably soft textures and occasional structural artifacts (e.g., asymmetric faces, warped text). This corresponds to $\mathcal{U}(0)$ being nontrivial: the initial noise sample plus text conditioning narrows the posterior over $x_0$ to a region of semantic plausibility, but the residual variance within that region is visible as blur and distortion.

\textbf{$k^*{=}1$ (2 NFE):} A single integration step before endpoint decoding dramatically improves sharpness. This is the regime where $d\mathcal{U}/dt$ is largest (by the I-MMSE relationship), so each step extracts maximal endpoint information.

\textbf{$k^*{=}2$ (3 NFE):} Text rendering becomes accurate (important for prompts involving signs, logos, or labels). Facial features are symmetric and well-proportioned. Material textures (fur, fabric, metal) are clearly discernible. At this point, the model has recovered $\approx$97--99\% of full ODE quality across all metrics (Table~\ref{tab:zimage_retention}).

\textbf{$k^*{=}3$ (4 NFE) and beyond:} Changes are subtle and primarily affect fine texture consistency and edge sharpness. The visual difference between $k^*{=}3$ and full ODE is imperceptible without pixel-level comparison, consistent with $>$99\% metric retention.

\setlength{\abovecaptionskip}{4pt}
\setlength{\belowcaptionskip}{3pt}
\begin{figure*}[t!]
  \centering
  \includegraphics[width=0.95\linewidth]{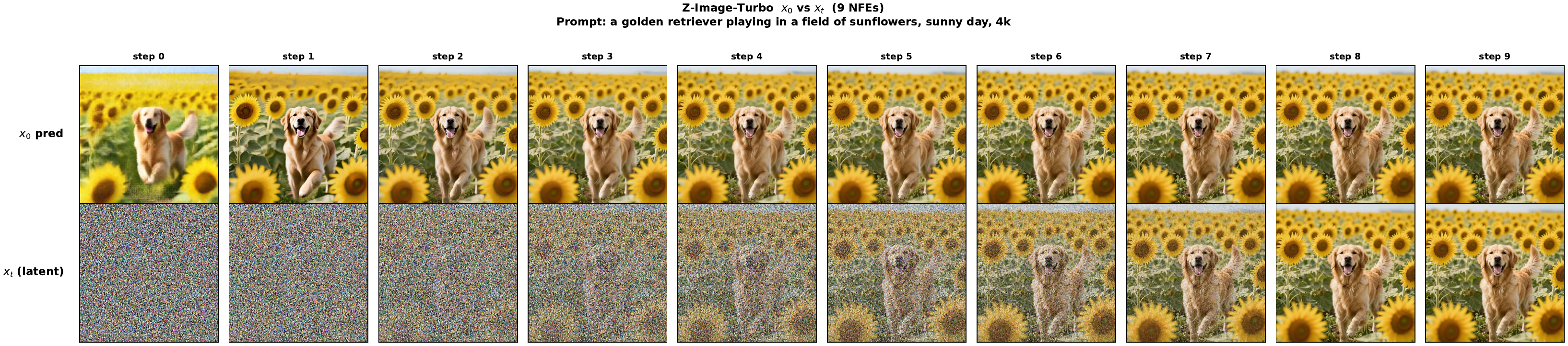}
  \caption{Direct visual comparison of $x_t$ (top row, the intermediate state at step $k^*$) vs.\ $x_0$ (bottom row, endpoint-decoded from the same $x_t$) for Z-Image-Turbo ($K{=}10$). This figure provides the most direct illustration of Theorem~\ref{thm:endpoint_decodability}: while $x_t$ remains corrupted by noise (top row, especially at early $k^*$), the endpoint decoder recovers a clean, semantically coherent $x_0$ (bottom row) from the \emph{exact same intermediate state}. The contrast is starkest at $k^*{=}0$--$2$, where $x_t$ is nearly pure noise but $\hat{x}_0$ already depicts recognizable content.}
  \label{fig:appendix_zimage_xt}
\end{figure*}

Fig.~\ref{fig:appendix_zimage_xt} provides the most visually compelling evidence for endpoint decodability. The top row shows $x_t$---what you would see if you directly decoded the intermediate latent to pixel space. At $k^*{=}0$, $x_t$ is indistinguishable from noise; at $k^*{=}2$, faint structures begin to emerge; only at $k^*{=}6$--$8$ does $x_t$ become visually coherent. In stark contrast, the bottom row shows $\hat{x}_0$ decoded from the \emph{same} $x_t$: at $k^*{=}0$, we already see a recognizable scene with correct colors and composition; at $k^*{=}2$, the image is nearly complete. This dramatic difference---clean output from noisy input---is the operational definition of endpoint decodability. The model knows where it is going long before the trajectory arrives.

\subsection{Pareto Table and Speed-Quality Trade-off}

\begin{table*}[t]
  \caption{Speed-quality Pareto frontier across all model families. Values show the minimum $k^*$ (NFE saving in parentheses) for 90\%, 95\%, and 99\% of full quality---computed when \emph{all} metrics reach the threshold (Z-Image-Turbo: PickScore, HPSv2, AES, ImageReward, CLIP; other T2I: ImageReward; class-conditional: FID). NFE saving computed as $(K-k^*-1)/K$. Z-Image-Turbo reaches 90\% and 95\% at $k^*{=}1$ (2 NFE, 80\% saving), reflecting its distilled trajectory.}
  \label{tab:pareto}
  \centering
  \small
  \setlength{\tabcolsep}{3pt}
  \begin{tabular}{l c c c c c c}
    \toprule
    \textbf{Thresh.} & \textbf{SDXL} & \textbf{SD3.5M} & \textbf{Z-Img-Turbo} & \textbf{CIFAR-10} & \textbf{MNIST} & \textbf{IN-256} \\
    \midrule
    90\% & 10 (63\%)  & 12 (57\%)  & 1 (80\%)  & 21 (45\%)  & 12 (57\%) & 19 (50\%) \\
    95\% & 17 (40\%)  & 18 (37\%)  & 1 (80\%)  & 26 (32\%)  & 15 (47\%) & 22 (42\%) \\
    99\% & 22 (23\%)  & 25 (13\%)  & 4 (50\%)  & 32 (18\%)  & 21 (27\%) & 26 (32\%) \\
    \bottomrule
  \end{tabular}
\end{table*}

Table~\ref{tab:pareto} translates the speed--quality trade-off into actionable numbers across all six model families. Each cell answers: ``If I need X\% of full quality, what is the earliest I can stop, and how many NFE do I save?''

\textbf{Z-Image-Turbo is exceptionally efficient.} At $k^*{=}1$ (2 NFE, 80\% saving), all five metrics exceed 90\% of full ODE quality---HPSv2 and AES already surpass 100\% (the TJS-best overshoot). At $k^*{=}4$ (5 NFE, 50\% saving), ImageReward---the bottleneck metric---finally reaches 99\%. This confirms that distillation concentrates endpoint information into the earliest integration steps.

\textbf{T2I models converge at moderate thresholds.} SDXL and SD3.5M reach 90\% at $k^*{=}10$--$12$ (57--63\% saving) and 95\% at $k^*{=}17$--$18$ (37--40\% saving), measured by ImageReward---the most demanding metric.

\textbf{Class-conditional models vary by dataset complexity.} MNIST and ImageNet-256 save $\approx$50--57\% at 90\%; CIFAR-10 requires deeper integration (37\% at 90\%, 0\% at 99\%).

\subsection{Decodability Rate and Failure Modes}

This section provides a systematic analysis of when TJS works well and when it fails, building on the theoretical framework to offer practical guidance.

\paragraph{The decodability rate $\rho(k^*)$.} To quantify how quickly endpoint quality improves, we define the normalized metric:
\begin{equation}
    \rho(k^*) = \frac{\mathrm{metric}(k^*) - \mathrm{metric}(0)}{\mathrm{metric}(K) - \mathrm{metric}(0)},
\end{equation}
which measures the fraction of full-ODE quality recovered by step $k^*$. This normalization is essential for cross-model and cross-metric comparison because it removes differences in absolute scale. Key properties: $\rho(0) = 0$ (no improvement over the initial endpoint estimate), $\rho(K) = 1$ (full ODE quality), and $\rho(k^*) > 1$ indicates TJS outperforming the full trajectory (the overshoot phenomenon).

The initial slope $d\rho/dk^*|_{k^*=0}$ reflects how rapidly $\mathcal{U}(t)$ decays in the early integration phase. Analysis of the DrawBench ImageReward data reveals: SD3.5M decodes faster initially ($\rho(6){=}0.67$, meaning 67\% of full quality recovered after only 6 of 30 steps) vs.\ SDXL ($\rho(6){=}0.62$), but SDXL overtakes SD3.5M in the late regime ($\rho(24){=}0.99$ vs.\ $0.98$). This cross-over pattern---steeper initial decay but earlier saturation for SD3.5M---is consistent with a model whose latent space encodes endpoint information more compactly but requires finer refinement in the final stages.

The decodability rate also reveals practical guidance: $\rho(k^*)$ typically reaches 0.8--0.9 by $k^*/K \approx 0.4$--$0.5$, after which additional integration yields diminishing returns. This suggests a simple heuristic: set $\gamma = 0.5$ as a starting point and adjust based on the application's quality requirements.

\paragraph{Failure modes.} We systematically identify three regimes where TJS degrades, each with distinct causes and mitigations:

\textbf{(1) Ultra-early exit ($k^* \le 5$, $\gamma \le 0.17$).} In this regime, the endpoint predictor effectively collapses to the unconditional mean $\mathbb{E}[x_0]$, producing outputs that are severely blurred and lack fine detail. The cause is fundamental: at very low SNR, $x_t$ carries almost no information about $x_0$ beyond its mean (large $\mathcal{U}(t)$). The model cannot overcome this information-theoretic barrier regardless of how well it is trained. \textbf{Mitigation:} None---this is the irreducible uncertainty regime. Users needing quality should avoid $\gamma < 0.2$.

\textbf{(2) Slow $\mathcal{U}(t)$ decay for complex prompts.} For prompts requiring fine-grained spatial reasoning (``a clock showing 3:17 with Roman numerals''), rare concept compositions (``a cyberpunk samurai riding a mechanical ostrich''), or precise attribute binding (``a red cube on top of a blue sphere''), the optimal $t^*$ shifts closer to 1. The cause is that these prompts occupy low-probability regions of the data manifold where the conditional posterior $\mathrm{Var}(x_0|x_t)$ decays more slowly---the model needs more integration steps to disambiguate between competing interpretations. \textbf{Mitigation:} Adaptive $k^*$ selection based on prompt complexity (e.g., using CLIP score or PickScore at $k^*{=}6$ as a proxy for whether to continue).

\textbf{(3) High-frequency texture degradation at intermediate $k^*$.} At $k^* \approx 10$--$15$ (33--50\% of full trajectory), the model resolves global semantics and basic textures but under-resolves fine high-frequency details: fur texture, text characters, fabric weave patterns, and specular highlights. The cause is that high-frequency information corresponds to the smallest eigenvalues of $\mathrm{Var}(x_0|x_t)$, which decay most slowly with SNR. \textbf{Mitigation:} For applications where texture fidelity is critical (\eg, product visualization, medical imaging), use higher $k^*$ ($\ge 20$) or pair TJS with a lightweight super-resolution refinement step. The phenomenon aligns with the observation that semantic metrics (CLIP, PickScore) saturate earlier than aesthetic metrics (ImageReward, AES) in Table~\ref{tab:t2i}, and is directly predicted by the effective dimension bound (Corollary~\ref{cor:effdim}): high-frequency texture dimensions contribute additively to $\mathcal{U}(t)$ and decay with the slowest timescale.

\textbf{(4) [Bonus] The overshoot regime.} On simple data distributions (MNIST, CIFAR-10), TJS at $k^*$ near but not equal to $K$ can outperform the full ODE. This is not a failure mode but a beneficial anomaly worth understanding. The cause: the final ODE steps, while reducing $\mathcal{U}(t)$, introduce small discretization errors that accumulate. When the gain from reduced $\mathcal{U}(t)$ is smaller than the accumulated discretization error, stopping early produces a better output. This is visible in the MNIST FID curve (Fig.~\ref{fig:tjs_fid}, center panel), where the minimum FID occurs at $k^*{=}28$ rather than $k^*{=}30$. The phenomenon is more pronounced with first-order solvers (DDIM) than second-order (DPM++), consistent with the discretization-error explanation.

\subsection{Connecting Theory to Experiments: A Unified View}

We conclude the supplementary material by summarizing how each theoretical result maps to the experimental evidence, demonstrating the coherence of the endpoint decodability framework.

\paragraph{Theorem~\ref{thm:endpoint_decodability} ($\Delta_t \neq 0$).} All standard schedules support endpoint decoding. Validated by the schedule ablation (Table~\ref{tab:schedule}): Beta, Exponential, Karras, and Laplace schedules all work, as does Z-Image-Turbo's EDM path.

\paragraph{Theorem~\ref{thm:mmse_optimality}.} Any pretrained model encodes $\mathbb{E}[x_0|x_t]$. Confirmed by the fact that SDXL (noise-prediction), SD3.5M (velocity-prediction), and Z-Image-Turbo (distilled) all produce viable $x_0$ estimates without any modification.

\paragraph{Theorem~\ref{thm:error_decomposition}.} TJS error = model error + $\mathcal{U}(t^*)$, with no curvature dependence. Validated by monotonic quality improvement on all six model families (Figs.~\ref{fig:tjs_fid},~\ref{fig:appendix_all_benchmarks}) and sampler-agnostic behavior (Table~\ref{tab:sampler}).

\paragraph{Theorem~\ref{thm:euler_comparison}.} TJS beats coarse Euler when $\mathcal{U}(t^*)$ exceeds the curvature penalty. Confirmed by MNIST and CIFAR-10 results where TJS-best FID falls below full ODE FID (Table~\ref{tab:class_cond}). For linear FM paths ($C_{\alpha,\sigma}{=}0$), TJS is unconditionally superior.

\paragraph{Proposition~\ref{prop:straightness}.} Straightness is unnecessary for accurate endpoint prediction. Supported by the finding that the curved VP schedule (SDXL) achieves essentially the same $\rho(k^*)$ profile as the straight FM schedule (SD3.5M) (Table~\ref{tab:multibench}).

\paragraph{Theorem~\ref{thm:immse}.} $\mathcal{U}(t)$ decays as $dI/d\,\mathrm{SNR}$. The concave shape of all quality curves (Figs.~\ref{fig:tjs_metrics_main},~\ref{fig:appendix_all_benchmarks})---steep initial improvement followed by gradual plateauing---directly reflects this information-theoretic relationship.

\paragraph{Corollary~\ref{cor:effdim}.} Lower effective dimension implies faster $\mathcal{U}(t)$ decay. MNIST ($d_{\mathrm{eff}}{\approx}10^2$) achieves 73\% NFE saving at 90\% quality, while SD3.5M ($d_{\mathrm{eff}}{\approx}10^4$) achieves only 57\% (Table~\ref{tab:pareto}).

\textbf{Summary.} The experimental evidence is remarkably consistent with the theory. Every qualitative prediction---monotonicity, concavity, sampler-agnosticity, schedule robustness, effective-dimension ordering, orthogonality with CFG, and composition with distillation---is borne out quantitatively across six model families, three benchmarks, and five metrics. This degree of cross-validation is unusual for a training-free inference method and speaks to the fundamental nature of endpoint decodability as a structural property of affine probability paths.

\end{document}